\documentclass[11pt]{article}
	
	\newcommand{\blind}{0}
	
    \makeatletter
    \renewcommand\section{\@startsection {section}{1}{\z@}%
                                       {-3.5ex \@plus -1ex \@minus -.2ex}%
                                       {2.3ex \@plus.2ex}%
                                       {\normalfont\fontfamily{phv}\fontsize{16}{19}\bfseries}}
    \renewcommand\subsection{\@startsection{subsection}{2}{\z@}%
                                         {-3.25ex\@plus -1ex \@minus -.2ex}%
                                         {1.5ex \@plus .2ex}%
                                         {\normalfont\fontfamily{phv}\fontsize{14}{17}\bfseries}}
    \renewcommand\subsubsection{\@startsection{subsubsection}{3}{\z@}%
                                        {-3.25ex\@plus -1ex \@minus -.2ex}%
                                         {1.5ex \@plus .2ex}%
                                         {\normalfont\normalsize\fontfamily{phv}\fontsize{14}{17}\selectfont}}
    \makeatother
	
	\usepackage{amsmath}
    \usepackage{amssymb,amsthm}
	\usepackage{graphicx}
	\usepackage{enumerate}
	\usepackage{natbib} 
	\usepackage{url} 
    \usepackage{xcolor}
    \usepackage{subcaption}
    \usepackage{algorithm}
    \usepackage{algpseudocode}
    \usepackage{float}
    \usepackage{booktabs}
    \usepackage{threeparttable}
    \usepackage{multirow}
    \usepackage{longtable}
    \usepackage{booktabs}
    \usepackage[most]{tcolorbox} 
    \usepackage{fvextra}         
    \usepackage{xcolor}         
    \usepackage{ragged2e}
    \usepackage{fvextra}
    \usepackage{multirow}
    \usepackage{enumitem}
    \usepackage[colorlinks=true,linkcolor=blue,citecolor=blue,urlcolor=blue]{hyperref}
    \newtheorem{theorem}{Theorem}[section]
    \newtheorem{definition}{Definition}[section]
    
    \usepackage{adjustbox}
    \algrenewcommand\algorithmicrequire{\textbf{Inputs:}}
    \algrenewcommand\algorithmicensure{\textbf{Outputs:}}


\begin{document}
		
		\def\spacingset#1{\renewcommand{\baselinestretch}%
			{#1}\small\normalsize} \spacingset{1}

    \if0\blind
    {
      \title{LLM-Driven Joint Evolution of Coupled Heuristics Components for Routing Optimization}
      }
      \author{%
        Juntao Wei$^{a}$,
        Yangming Zhou$^{a}$,
        Zhibin Jiang$^{a,*}$,
        Shan Jiang$^{b}$\\[5pt]
        \small
        \begin{tabular}{@{}l@{}}
        $^{a}$ Antai College of Economics and Management,
        Shanghai Jiao Tong University, Shanghai 200030, China\\
        $^{b}$ Global Institute of Future Technology,
        Shanghai Jiao Tong University, Shanghai 200240, China
        \end{tabular}
        }
      \date{}
      \maketitle
    \fi

		\if1\blind
		{

            \title{\bf \emph{IISE Transactions} \LaTeX \ Template}
			\author{Author information is purposely removed for double-blind review}
			
\bigskip
			\bigskip
			\bigskip
			\begin{center}
				{\LARGE\bf \emph{IISE Transactions} \LaTeX \ Template}
			\end{center}
			\medskip
		} \fi
	\begin{abstract}

Heuristic design for combinatorial optimization remains heavily reliant on expert knowledge, while existing large language model (LLM)-enhanced evolutionary methods typically evolve isolated algorithmic components, even when one determines the search state on which another operates. This paper proposes LLM-driven Heuristic Components Joint Generation (LLM-HCJG), a population-based framework that jointly generates and co-evolves interdependent heuristic components under a shared design blueprint. Applied to guided local search (GLS), LLM-HCJG couples solution initialization with penalty construction and embeds the generated pair into an enhanced online search mechanism. The resulting form is further transferred from the traveling salesman problem (TSP) to the capacitated vehicle routing problem (CVRP). Theoretical analysis establishes the non-separable state-transition effects between the two components and the advantage in generation consistency. Across synthetic instances and 41 public TSPLIB/CVRPLIB benchmarks, LLM-HCJG attains consistently low optimality gaps, including best or tied-best results on 28 of 29 TSPLIB instances and all 12 CVRPLIB instances. Ablation and structural analyses further indicate that these gains are associated with cross-component compatibility and alignment rather than isolated-component recombination. These results support effective cross-instance transfer within the evaluated routing settings under limited-sample, modest-cost training.

	\end{abstract}
			
	\noindent%

	\spacingset{1.5} 

\section{Introduction}
\label{Sec:Introduction}
Combinatorial optimization problems are ubiquitous in logistics and transportation systems \citep{song2025dualsourcing_tracking,cheng2025robust_crowdsourced_delivery}. Such problems require repeated generation of high-quality routing or scheduling decisions under stringent operational constraints, often in dynamic and uncertain environments where solution quality directly impacts operational efficiency and service performance.

As problem scale and structural complexity increase, exact algorithms become computationally infeasible \citep{zhen2023branch}, rendering heuristic algorithms indispensable in practice \citep{blum2003metaheuristics}. However, achieving strong performance with these methods typically requires extensive domain knowledge and carefully handcrafted rules, which can limit adaptability and lead to premature convergence \citep{bengio2021machine}. Meanwhile, many recent learning-based approaches are highly task-specific and often require substantial retraining when problem settings change \citep{kim2024multi}. Consequently, a persistent gap remains between the complexity of real-world environments and the practical deployability of existing solution methods, hindering the development of truly adaptive and scalable optimization frameworks.

Recent advances in large language models (LLMs) expand optimization research through improved reasoning, code generation, and knowledge integration \citep{vaswani2017attention, guo2025deepseek}. A potential paradigm is emerging, where LLMs serve not only as problem solvers but also as algorithm designers to address the limits of human intuition and handcrafted engineering \citep{huang2024large}. In parallel, combining LLMs with evolutionary computation (EC) yields LLM-enhanced EC, in which prompt-based EC drives LLMs to generate and evolve program codes, making for algorithm-level search beyond conventional solution optimization \citep{wu2024evolutionary}.

Currently, the main focus is on embedding highly competitive LLM-based components into fixed heuristics to iteratively improve existing algorithms \citep{zhang2024understanding}. Following this trend, this paper investigates an LLM-driven guided local search (GLS) framework. GLS alternates between neighborhood-based improvement and penalty-driven perturbation, with the latter enabling escape from local optima and encouraging exploration. Typically, two bottleneck functions jointly determine the overall search direction: (i) solution initialization: constructing the starting route and strongly influencing the subsequent search trajectory; and (ii) penalty construction: defining distance matrix update rules that allow the search to escape local minima while preserving solution structure. These two inherently interdependent components traditionally render substantial domain expertise essential to effective design.

To address these limitations, we propose the \textbf{LLM}-driven \textbf{H}euristic \textbf{C}omponents \textbf{J}oint \textbf{G}eneration (LLM-HCJG) framework. It advances a distinct dimension of LLM-based heuristic design by explicitly coordinating search-structurally interdependent components (solution-initialization and penalty-construction functions) within a fixed heuristic framework through a shared design blueprint and joint evolution. The main contributions of this work are summarized as follows:
\begin{itemize}
\item \textbf{Joint algorithm-level representation for LLM-driven heuristic design.} LLM-HCJG represents each individual as a complete heuristic organized around a shared design blueprint and composed of multiple coupled components, rather than as a solution or an isolated rule. This enables algorithm-level search while explicitly modeling the coevolution of interdependent components of search direction.

\item \textbf{Theoretical foundation for coupled component generation.} We establish that solution initialization and penalty construction are non-separable in the search-state transition, justifying the complete component pair as the basic unit of the evolution loop. Under the blueprint-consistency condition, the proposed joint-generation mechanism enforces design-principle alignment and reduces mismatch probability.

\item \textbf{Search-direction redesign of GLS from TSP to CVRP.} The jointly generated components are embedded into a framework-augmented GLS procedure. This design preserves the interpretable backbone while extending the algorithm from TSP to CVRP.

\item \textbf{Strong empirical performance under modest training cost.} On synthetic TSP/CVRP testbeds across the 20-, 50-, and 100-node scales and 41 public TSPLIB (size 51--200) / CVRPLIB (size 32--190) benchmarks, LLM-HCJG confirms the co-evolved structural compatibility via attaining the lowest gaps and best counts of 28/29 and 12/12. These results rely on only five training instances per fitness evaluation, demonstrating that strong heuristics can be generated with limited samples and modest cost.
\end{itemize}

The remainder of this paper is organized as follows. Section~\ref{Sec:Related Work} reviews representative paradigms of LLM-based optimization and highlights the bottlenecks of existing LLM-enhanced evolutionary design. Section~\ref{Sec:Methods} develops LLM-HCJG as a GLS search-direction redesign framework, where coupled component generation and enhanced online search jointly reshape the guided-search process from TSP to CVRP. Section~\ref{Sec:Experiments} characterizes the best-evolved heuristics and evaluates performance through synthetic testbeds, TSPLIB/CVRPLIB benchmarks, plus ablation analyses of component co-adaptation. Section~\ref{Sec:Conclusion} concludes the paper with a discussion of broader implications and future research directions.
\section{Related Work}
\label{Sec:Related Work}
From a dual-role perspective, LLMs contribute in two fundamentally different ways: in the generative role, LLMs directly produce candidate solutions through prompt interpretation; in the evolutionary role, LLMs iteratively refine search strategies based on real-time feedback and explicit operators \citep{yu2024deep}. Consequently, existing studies can be broadly categorized into two categories: LLMs as an automated solution solver and LLMs as an automated algorithm designer, as illustrated in Figures~\ref{fig:paradigm1} and \ref{fig:paradigm2}, respectively. Note that the latter paradigm is commonly referred to as LLM-enhanced EC, to which our work belongs, with a focus on heuristic innovation.
\begin{figure}[htbp]
    \centering
    \begin{subfigure}[b]{0.32\textwidth}
        \centering
        \includegraphics[width=\textwidth]{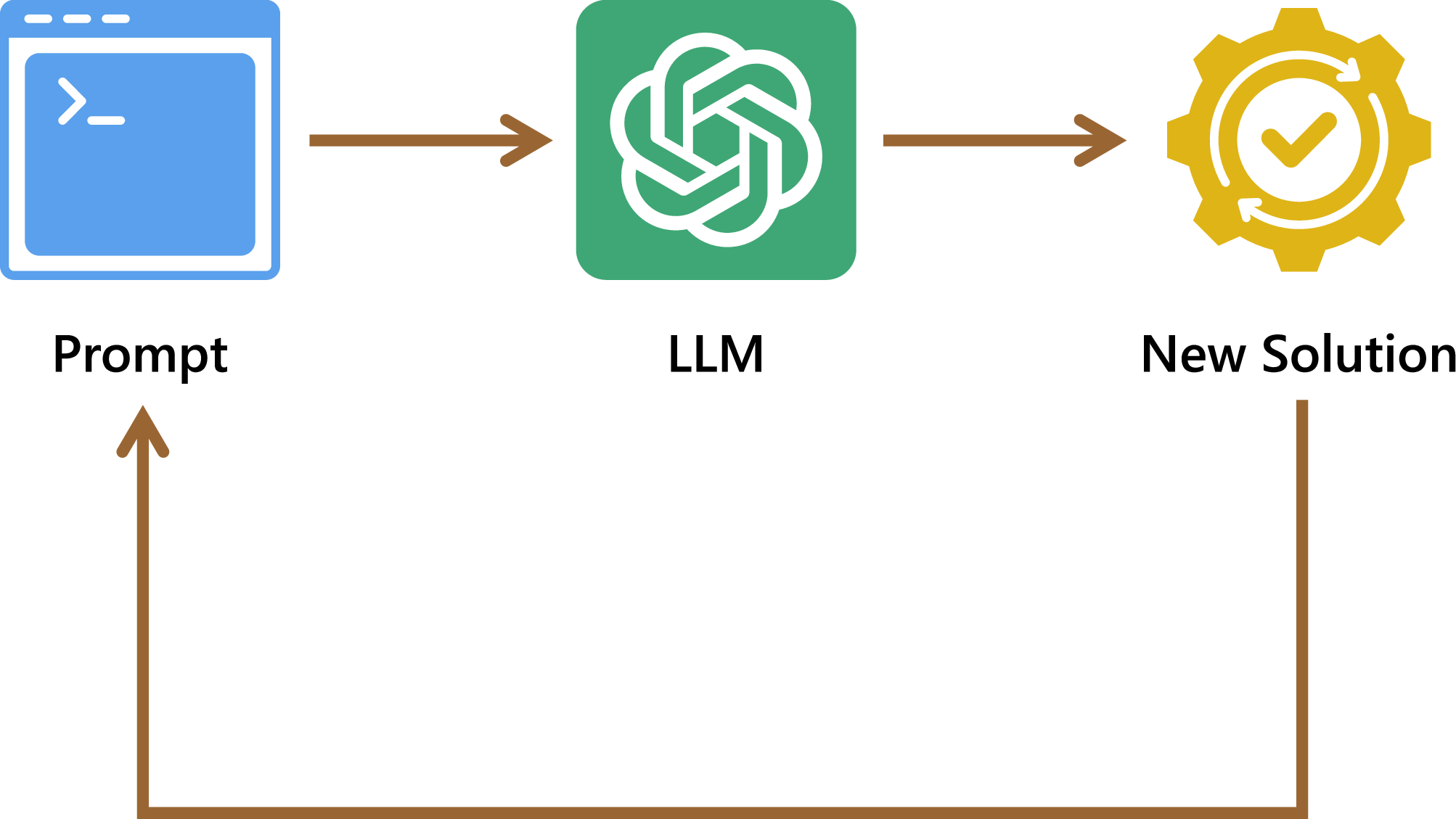}
        \caption{}
        \label{fig:paradigm1}
    \end{subfigure}
    \hspace{0.1\textwidth}
    \begin{subfigure}[b]{0.35\textwidth}
        \centering
        \includegraphics[width=\textwidth]{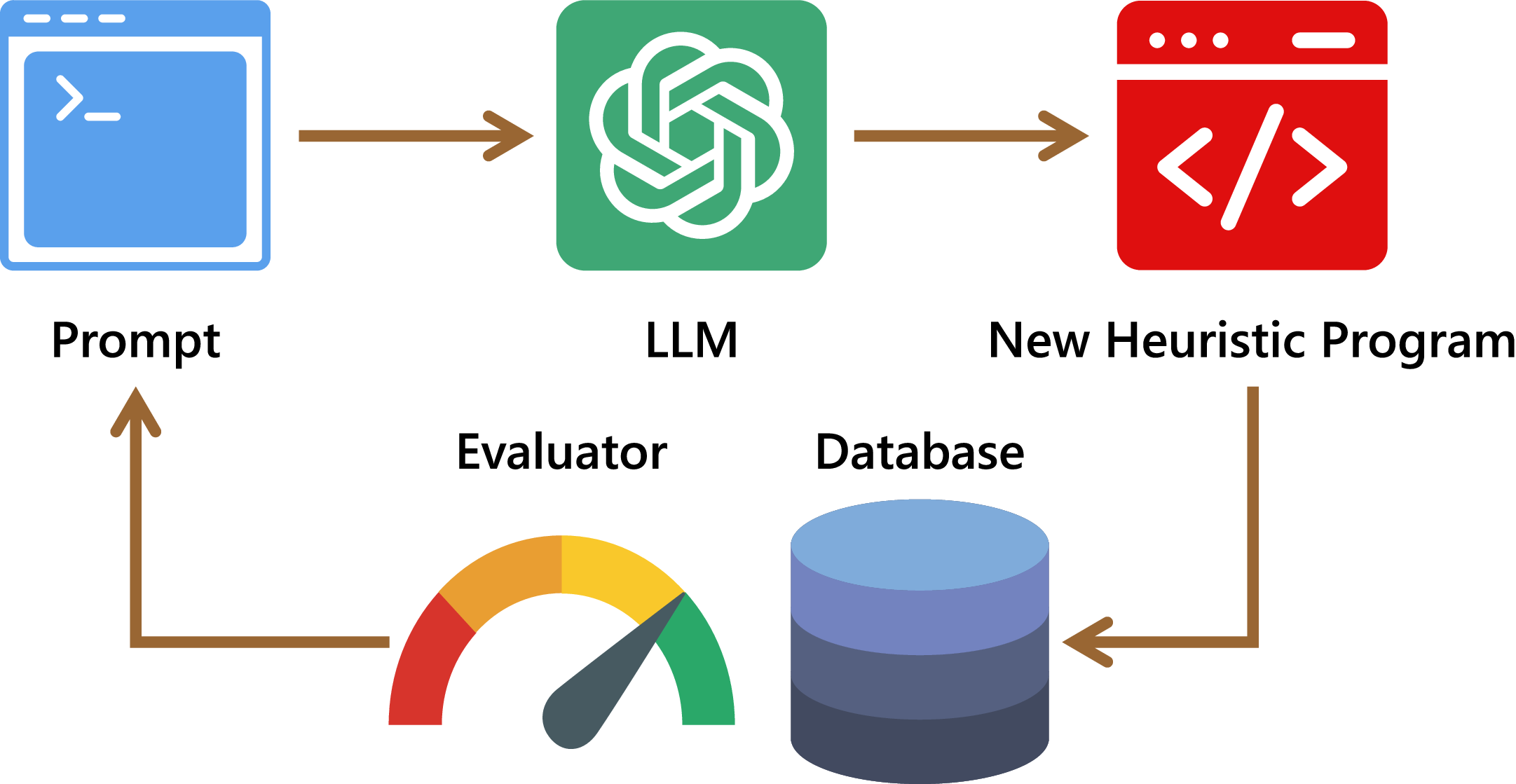}
        \caption{}
        \label{fig:paradigm2}
    \end{subfigure}
    \caption{(a) LLMs as a solution solver; (b) LLMs as an algorithm designer.}
    \label{fig:two_paradigms}
\end{figure}
\subsection{LLMs as an Automated Solution Solver}
\label{sec:slover}
This paradigm treats LLMs as black-box optimizers that directly produce one or more candidate solutions from predefined prompts. Early studies explore the iterative refinement of single solutions. \cite{yang2023large} propose Optimization by PROmpting, where meta-prompts encode solution history, corresponding objective values, and guidance for exploration direction. Similarly, \cite{huang2024words} show that LLMs can directly generate executable vehicle routes via self-reflection mechanisms that improve feasibility and quality. Subsequent work extends to population-level search. \cite{brahmachary2025large} introduce a dual-pool prompting framework that explicitly separates exploration and exploitation, exchanging high-quality candidates via a port-and-filter. Furthermore, \cite{liu2024large} embed a series of evolutionary operators in the prompts to form an LLM-driven population search with adaptive temperature control. Several studies augment solution generation by incorporating multimodal inputs. By merging visual-spatial information with textual constraints, \cite{huang2025multimodal} and \cite{elhenawy2024visual} demonstrate improved performance in CVRP and TSP variants, particularly in capturing geometric structures beyond text-only prompts.

Although this paradigm is attractive, it still suffers from inherent limitations, including high prompt sensitivity, performance degradation due to weak reasoning consistency, and the absence of feedback mechanisms for feasibility or improvement. These limitations indicate that relying solely on LLMs as end-to-end solvers is insufficient.
\subsection{LLMs as an Automated Algorithm Designer}
\label{sec:algorithm designer}
These limitations motivate a shift toward guiding LLMs to design innovative heuristic strategies, including partial or complete algorithms. Early studies adopt primarily a single-round approach, in which LLMs generate heuristics in one shot \citep{sartori2025combinatorial}. However, without performance-driven feedback or iterative refinement, the resulting algorithms function more like conversational than genuine optimization output, and \cite{zhang2024understanding} show that standalone LLMs generation is insufficient for algorithm design.

Motivated by this issue, the LLM-enhanced EC embeds LLM-generated heuristics into an evolutionary framework for iterative refinement. Automated heuristic design is formulated as an evolutionary program search problem, where heuristics are represented as executable programs, evolutionary mechanisms guide performance-driven variation and selection, and LLMs generate program variants \citep{zhang2024understanding}. One straightforward realization is single-trajectory refinement, where a candidate heuristic is repeatedly modified. The self-taught optimizer \citep{zelikman2024stop} integrates classical heuristic principles into a recursive refinement process of scaffold programs, producing transferable optimizers for COPs such as Max-Cut. However, the single-trajectory nature inherently restricts exploration of algorithmic diversity.

Recent research has shifted toward population-based evolutionary heuristics. A representative example is FunSearch \citep{romera2024mathematical}, which evolves heuristic code fragments through distributed island populations, prompt-based variation, and migration to maintain diversity. Subsequently, LLaMEA \citep{van2024llamea} enables LLMs to generate complete heuristic algorithms within a predefined interface and builds an automated evolutionary loop for execution, verification, filtering, and repair. Despite promising results, both rely exclusively on code-level evolution to identify effective heuristics, thereby incurring the high computational cost of LLM queries.

To overcome this shortcoming, EoH \citep{9e7eceac6e31432aafced5e02a49de16} generates population individuals directly from task-descriptive prompts, pairing a core design idea in natural language with a corresponding executable code implementation and evolving via prompt-based genetic operators. Subsequent extensions further broaden the applicability and efficiency of EoH. \cite{10945804} integrate EoH into a memetic framework for Lot-Streaming Hybrid Job Shop Scheduling with Variable Sublots, where EoH-designed heuristics function as decomposed subproblem-specific local search operators. \cite{wu2025efficient} propose Hercules, which introduces Core Abstraction Prompting to distill essential components from elite heuristics into prompt construction, thereby guiding a more focused search. The variant Hercules-P further incorporates Performance Prediction Prompting to estimate fitness from the semantic similarity of prior individuals, reducing redundant evaluations.

Recent studies suggest that LLM-enhanced evolutionary computation can benefit from richer feedback signals, motivating reflective evolutionary mechanisms. Reflective Evolution \citep{ye2024reevo} proposes short- and long-term reflections based on parental analysis and intergenerational comparison, incorporating both into prompts to steer future offspring. Building on this idea, self-evolution reflection \citep{huang2024automatic} is further incorporated to capture the characteristics of individual change induced by evolution. Since ReEvo struggles to maintain population diversity, HSEvo \citep{dat2025hsevo} adopts a more cost-effective flash reflection and harmony search to fine-tune elite parameters. Similar reflective principles also emerge in the dual-layer self-evolution structure of LLM-driven LNS \citep{ye2025large}, where the evolution of inner-layer destruction–repair heuristic operators informs revisions to outer-layer prompting strategies. 

Although these studies demonstrate the promise of LLM-enhanced EC, their main contributions remain largely centered on the generic automated heuristic-design paradigm. More broadly, these methods address the limited treatments of how LLM-designed components should interact with a specific optimization problem and algorithmic framework-level structure. Empirically, this limitation is most evident in EoH, whose GLS-TSP search takes nearly two days. Consequently, simply enriching the generic LLM-for-AHD loop with reflection, prompt refinement, or feedback augmentation may still leave the search direction under-specified when the target algorithm relies on tightly coupled components. More recently, CoEvo-AHD \citep{kuang2026llm} decomposes the problem decision into subproblem evolving populations. The coordination is only introduced through cross-population evaluation and recombination, leaving design-level interdependence among functional components within a single heuristic largely unexplored.
\section{Joint-Generation Approach for LLM-Enhanced GLS Design }
\label{Sec:Methods}
LLM-HCJG advances a distinct algorithmic dimension by redesigning search direction within the interpretable GLS backbone for both TSP and CVRP. To preserve controllability and verifiability, LLM-driven modifications are confined to two locally critical components that directly shape search bias. Solution initialization determines the starting search state, whereas penalty construction guides subsequent perturbations. Accordingly, the following subsections first introduce the LLM-HCJG framework, then present edge-penalized GLS from TSP to CVRP, and finally describe joint component generation with enhanced guided search.
\subsection{LLM-Driven Heuristic Components Joint Generation Framework}
\label{LLMGLS}
Our framework relies on two fundamental mechanisms:
\begin{itemize}[leftmargin=*]
\item \textbf{Evolutionary Representation in Algorithm Space.}  
Each individual is represented as a complete algorithm $a_j$ stored a design blueprint $Z_j$ and coupled component functions ${F}^{\text{m}}_j$, rather than as a candidate solution. 
\item \textbf{LLM-driven Construction and Refinement of Algorithmic Components.}  
Unlike costly search procedures and domain-specific model retraining, LLMs are integrated into the evolutionary cycle to automatically generate and refine algorithmic components, enabling more efficient exploration.
\end{itemize}

The LLM-HCJG framework operationalizes a population-based evolutionary process in which multiple pairs of interdependent heuristic component functions interact. The complete procedure is summarized in Algorithm~\ref{alg:LLM-HCJG}. At initialization, a population $P$ consists of $N$ candidate individuals created by LLMs, and each $a_j$ is evaluated in training instances to obtain fitness. The population then evolves over $N_g$ generations via selection, crossover, mutation, and population management, until the best-performing algorithm $a_j^*$ is identified. Crossover and mutation are activated stochastically with probabilities $\sigma_1$ and $\sigma_2$, respectively, to adjust the evolutionary intensity. The detailed framework follows in the subsequent subsections.
\begin{algorithm}[!ht]
\footnotesize
\caption{LLM-driven Heuristic Components Joint Generation Framework}\label{alg:LLM-HCJG}
\begin{algorithmic}[1]
  \Require Number of generations $N_g$, population size $N$, crossover probability $\sigma_1$, mutation probability $\sigma_2$, number of parents for crossover $l$, number of offspring per crossover $s$ and a given universal LLM
  \Ensure Best algorithm $a_j^*$

  \State $P \gets \textbf{Initialization}(\text{LLM},\, i1, N)$, $\mathcal{F} \gets \textbf{Evaluate}(P)$

  \For{$i = 1$ \textbf{to} $N_g$}
   \State $P_{\text{off}} \gets \emptyset$
      \State Sample $u_1 \sim \mathcal{U}(0,1)$ 
    \If{$u_1 < \sigma_1$} \hfill \texttt{// Crossover with probability $\sigma_1$}
      \For{$k = 1$ \textbf{to} $N$}
        \State $P_{\text{par}} \gets \textbf{Select}(P,\, l)$ 
        \State $P_{\text{cro}} \gets \textbf{Crossover}(\text{LLM},\, e1, e2,\, P_{\text{par}},\, s)$, $\mathcal{F}_{\text{cro}} \gets \textbf{Evaluate}(P_{\text{cro}})$
        \State $P_{\text{off}} \gets P_{\text{off}} \cup P_{\text{cro}}$
      \EndFor
    \EndIf

      \State Sample $u_2 \sim \mathcal{U}(0,1)$
    \If{$u_2 < \sigma_2$} \hfill \texttt{// Mutation with probability $\sigma_2$}
      \For{$k = 1$ \textbf{to} $N$}
        \State $P_{\text{par}} \gets \textbf{Select}(P,\, 1)$
        \State $P_{\text{mut}} \gets \textbf{Mutate}(\text{LLM},\, m1, m2, m3,\, P_{\text{par}},\, 1)$, $\mathcal{F}_{\text{mut}} \gets \textbf{Evaluate}(P_{\text{mut}})$
        \State $P_{\text{off}} \gets P_{\text{off}} \cup P_{\text{mut}}$
      \EndFor
    \EndIf
    \State $P \gets \textbf{Management}(P \cup P_{\text{off}}, N)$
  \EndFor
  \State $a_j^* \gets \text{ArgBest}(P)$
\end{algorithmic}
\end{algorithm}
\subsubsection{Individual Representation in Joint coevolution Mechanism}
Three layers (see Figure~\ref{fig:individual}) together define what the individual represents, how component functions are organized, and how the evolutionary process is driven within the LLM-HCJG framework.

\begin{figure}[htbp]
    \centering
    \includegraphics[width=0.6\textwidth]{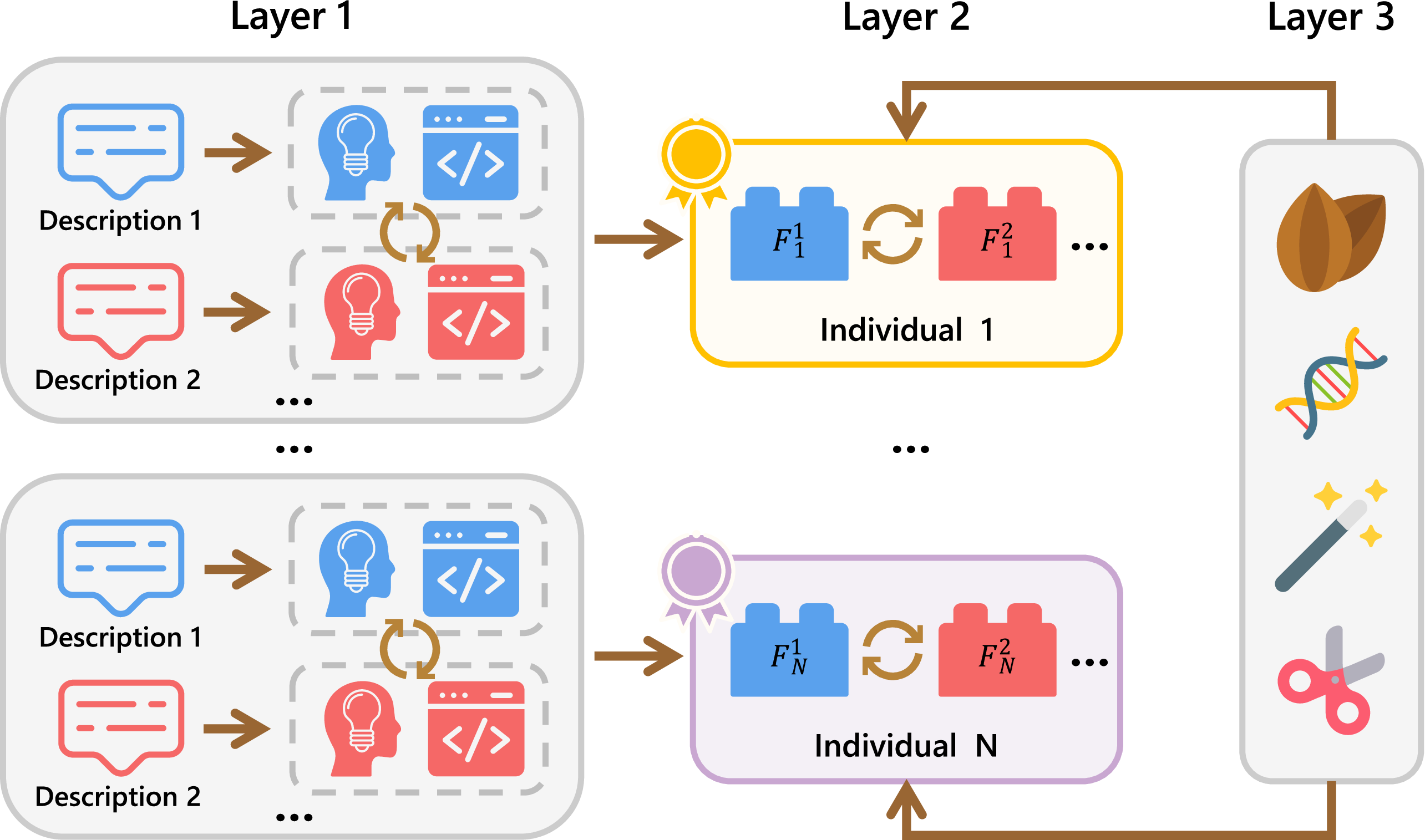}
    \caption{Illustration of the three-layer framework for heuristic modular composition.}
    \label{fig:individual}
\end{figure}

\textbf{Layer 1: Component-Specific Function Structure (see Appendices~\ref{apx:tsp-basic-attributes} and~\ref{apx:cvrp-basic-attributes}).} Each heuristic component function corresponds to a distinct task and serves as the smallest representational unit. For structural consistency, each component is instantiated from a standardized prompt template with six fields:
\begin{itemize}[leftmargin=*]
\item \textbf{Task Definition} (self.prompt\_task): objective and background of the task;
\item \textbf{Function Name} (self.prompt\_func\_name): required function designation;
\item \textbf{Inputs} (self.prompt\_func\_inputs): input parameter list;
\item \textbf{Outputs} (self.prompt\_func\_outputs): expected outputs or return variables;
\item \textbf{I/O Specification} (self.prompt\_inout\_inf): input/output types, dimensions, ranges of values, and so on;
\item \textbf{Additional Requirements} (self.prompt\_other\_inf): extra constraints for correctness and robustness.
\end{itemize}

\textbf{Layer 2: Individual Composition (as shown in Appendix~\ref{apx:tsp-elite} and ~\ref{apx:cvrp-elite}).} This layer assembles Layer~1 component functions and shared coordination into one candidate individual: 
\begin{equation}
a_j = \big(Z_j,{F}^{\text{1}}_j,{F}^{\text{2}}_j,\dots,{F}^{\text{M}}_j\big),
\end{equation}
where 
\begin{equation}
{F}^{\text{m}}_j=\big({Thought}^{\text{m}}_j,{Code}^{\text{m}}_j\big)
\end{equation}
Each $a_j$ contains four elements:
\begin{itemize}[leftmargin=*]
\item \textbf{Shared blueprint $Z_j$:}  a one-sentence semantic description of the joint principle governing the complementary roles and their intended interaction;
\item \textbf{Thought ${Thought}^{\text{m}}_j$:} the component-level rationale that instantiates $Z_j$ for ${F}^{\text{m}}_j$;
\item \textbf{Code implementation ${Code}^{\text{m}}_j$:} the rogram executable realization of ${Thought}^{\text{m}}_j$;
\item \textbf{Fitness evaluation:} the performance score $\mathcal{F}(a_j)$ obtained by executing on training instances.
\end{itemize}

\textbf{Layer 3: Evolutionary Prompt Construction (as shown in \ref{apx:i1}-\ref{apx:m3}).} Building on the first two layers, this layer defines the prompts that control initialization, crossover, and mutation in the evolutionary process. The evolutionary prompts include:
\begin{itemize}[leftmargin=*]
\item \textbf{Task description:} a problem statement paragraph that aggregates all task templates;
\item \textbf{Parent algorithm(s):} selected parent individuals for recombination;
\item \textbf{Evolution-specific hints:} guidance for generating diverse variants;
\item \textbf{Blueprint-guided coordination:}  instructions that identify structural dependencies and induce a shared design principle across component functions;
\item \textbf{Other hints:} constraints encouraging concise and necessary outputs.
\end{itemize}

\subsubsection{Shared Design Blueprint in Joint LLM Generation}
\label{sec:shared-blueprint}
To make the intended association between coupled components explicit, we introduce an individual-level shared design blueprint. Rather than referring to the common task prompt itself, the blueprint is a concise joint design principle generated for one algorithmic individual and recorded together with its component functions. Let $\mathcal{Z}$ denote a finite blueprint space, and let $x$ collect the task description, function interfaces, input--output constraints, and component-association requirements. The LLM first generates a shared blueprint
\begin{equation}
Z_j \sim \pi_{\theta}(\cdot \mid x),
\label{eq:blueprint-distribution}
\end{equation}
and then generates the initialization and penalty-construction components conditionally on the same blueprint:
\begin{equation}
F_j^{s} \sim q_{\theta}^{s}(\cdot \mid x,Z_j),
\qquad
F_j^{u} \sim q_{\theta}^{u}(\cdot \mid x,Z_j).
\label{eq:conditional-component-generation}
\end{equation}

Let $\zeta^{s}$ and $\zeta^{u}$ denote blueprint-label mappings obtained by structured validation of the initialization and penalty components, respectively. We impose the blueprint-consistency condition
\begin{equation}
\Pr\!\left(\zeta^{s}(F^{s})=Z\mid Z\right)=1,
\qquad
\Pr\!\left(\zeta^{u}(F^{u})=Z\mid Z\right)=1.
\label{eq:blueprint-faithfulness}
\end{equation}
Operationally, the shared blueprint is parsed by a structured parser and propagated across initialization, crossover, mutation, evaluation, and selection. Structural check ensures explicit preservation and regenerates outputs with missing fields,  while semantic faithfulness remains the condition stated in \eqref{eq:blueprint-faithfulness}.

\begin{definition}[Blueprint-consistent component pair]
\label{def:blueprint-consistency}
Define the set of blueprint-consistent component pairs as
\begin{equation}
\mathcal{C}=\left\{(f^{s},f^{u}):\zeta^{s}(f^{s})=\zeta^{u}(f^{u})\right\}.
\label{eq:consistent-pair-set}
\end{equation}
A pair is design-principle consistent if and only if it belongs to $\mathcal{C}$; otherwise, it is blueprint-mismatched.
\end{definition}
\subsubsection{Initialization} 
\label{Initilization}
The initial population of size $N$ is built using prompt i1 (Appendix~\ref{apx:i1}). It instructs the LLM to act as an expert on the target problem and simultaneously design coupled heuristic component functions under shared blueprints.
\subsubsection{Crossover} 
\label{Crossover}
At each iteration $i$, crossover is activated with probability $\sigma_1$. When triggered, $l$ complete parent records form $P_{\text{par}}=\{a_1,\ldots,a_l\}$. Prompts e1 and e2 (Appendices~\ref{apx:e1}--\ref{apx:e2}) receive parents blueprint together with its component records. The offspring first establishes a child blueprint by either departing from or refining the parental principles, and then generates both components conditional on that blueprint. The $s$ offspring are generated per round, yielding  $sN$ new individuals. 
\subsubsection{Mutation} 
\label{Mutation}
Similarly, mutation is applied with probability $\sigma_2$. In each of the $N$ rounds, the selected parent’s blueprint and component records are supplied to one of prompts m1--m3 (Appendices~\ref{apx:m1}--\ref{apx:m3}), which respectively optimize performance, retune parameters, and prune redundant structure. The offspring inherits the blueprint if the joint principle is preserved and revises it otherwise. Each mutation yields a complete, explicitly recorded  individual.
\subsubsection{Parents Selection} 
\label{sec:selection}
The selection operator chooses \(l\) parents for crossover or one parent for mutation from \(P\) using rank-based roulette, where individuals are sorted by descending fitness \(\mathcal{F}(a_j)\) with ranks \(r_j\in\{0,\ldots,N-1\}\). The unnormalized selection weight is
\begin{equation}
w_j=\frac{1}{r_j+1+N},
\end{equation}
which assigns larger sampling probabilities to higher-fitness individuals while preserving stochasticity.
\subsubsection{Population Management} 
\label{sec:Population management}
Given the stochastic activation of crossover and mutation, the number of offspring per generation ranges from 0 to $(s+1)N$. To maintain a constant population size $N$ and selection pressure,  the combined pool is truncated by discarding inferior candidates and retaining the top $N$ individuals.
\subsection{\texorpdfstring{Generalizing the Guided Local Search Backbone from TSP to CVRP}{Generalizing the Guided Local Search Backbone from TSP to CVRP}}
\label{GLS}
\label{GLS}
Guided local search (GLS) \citep{voudouris1999guided} augments local search with feature penalties, making repeatedly visited structures progressively less attractive. Algorithm~\ref{alg:gls} presents the canonical edge-based GLS formulation for TSP, which serves as the reference mechanism for transferring the guided-search principle to CVRP.

The TSP solution is initialized with a nearest-neighbor tour and optimized following the same local search with iterative 2-opt and relocate moves in Lines~2 and~19. The search features are undirected tour edges, and each edge \(e_{ij}\) has a historical penalty count \(p_{ij}\) stored in the symmetric matrix \(\mathcal{P}\). At a local optimum, the utility of an edge is
\begin{equation}
u_{ij}(e_{ij})
=
I_{e_{ij}}
\frac{d_{ij}}{1+p_{ij}},
\label{eq:utility}
\end{equation}
where
\begin{equation}
I_{e_{ij}}
=
\begin{cases}
1, & e_{ij}\in\text{the current tour},\\[2pt]
0, & \text{otherwise}.
\end{cases}
\label{eq:indicator}
\end{equation}
This utility favors long edges penalized, while frequently penalizing edges are suppressed to prevent repeated concentration on the same feature. The edge with the largest utility is penalized, and the accumulated penalties define the guided distance matrix
\begin{equation}
\mathcal{D}^{\prime}
=
\mathcal{D}+\lambda\mathcal{P}
=
[\,d_{ij}+\lambda p_{ij}\,],
\label{eq:augmented}
\end{equation}
with
\begin{equation}
\lambda
=
\alpha\frac{cost^*}{n},
\label{eq:lambda}
\end{equation}
where \(cost^*\) is the incumbent-best tour cost, \(n\) is the number of customers, and \(\alpha\in[0,1]\) controls the penalty scale. Local search is performed under the augmented matrix \(\mathcal{D}^{\prime}\), while solution quality is always evaluated under the original matrix \(\mathcal{D}\). By discouraging recurrent edge structures and biasing the search for alternative solutions, GLS facilitates escape from local optima.

Extending GLS from TSP to CVRP is valuable because it tests whether the guided-search principle remains effective beyond a single unconstrained tour and under coupled routing and resource constraints. Unlike TSP, CVRP requires simultaneous decisions on customer assignment, route sequencing, and load feasibility, while solution must satisfy customer demand, vehicle capacity, and fleet-size constraints. Since no single standard GLS workflow is directly applicable to CVRP, we extend the core principle of edge-based penalization followed by local improvement from a single-tour search state to a capacity-constrained set of vehicle routes (Algorithm~\ref{alg:cvrp_gls} in Appendix~\ref{app:routing-gls-frameworks}). Three feasibility-preserving neighborhoods are act on local search stage: intra-route2-opt (Lines~2 and~18) for route sequencing, inter-route relocate (Lines~2 and~18) for customer reassignment, and inter-route swap (Line~18) for balancing route allocation under capacity constraints.
\subsection{\texorpdfstring{Joint Component Coevolution and Enhanced Guided Search}{Joint Component Coevolution and Enhanced Guided Search}}
\label{LLM-HCJG}
LLM-HCJG modifies the guided-search procedure at two distinct levels: 
(i) \textbf{component-design level}, where an LLM jointly generates the solution-initialization and penalty-construction components under a shared design blueprint; and (ii) \textbf{search-framework level}, where the generated components are embedded into a fixed online mechanism for influential-edge selection, targeted perturbation and repair, additional local improvement, and periodic restoration to the incumbent-best solution. Figure~\ref{fig:framework} summarizes this separation between joint component evolution and online algorithm execution.
\begin{figure}[!t]
\centering
\includegraphics[width=0.9\textwidth]{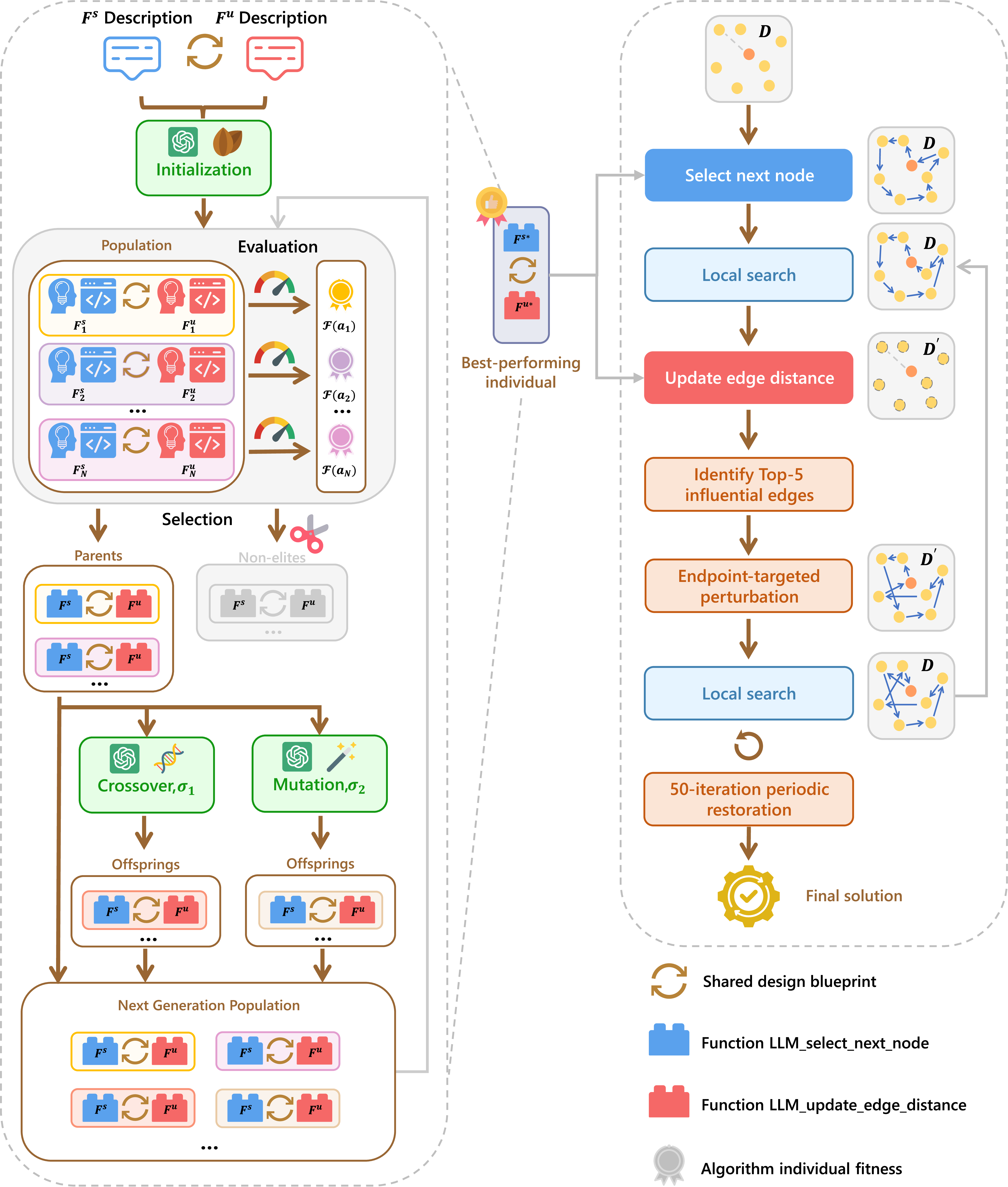}
\caption{LLM-HCJG framework: joint generation and enhanced guided search (TSP case).}
\label{fig:framework}
\end{figure}

\subsubsection{Joint Generation of Interdependent components}
The initialization component \(F_j^{s}\) implemented by \texttt{LLM\_select\_next\_node} constructs the search starting state, i.e., a TSP tour or a feasible CVRP route set. The penalty-construction component \(F_j^{u}\) maps this state and accumulated edge-usage information to a guided distance matrix through \texttt{LLM\_update\_edge\_distance}. For CVRP, both components also incorporate demand, capacity, and fleet size information to ensure route feasibility. The two functions are generated and evolved as an ordered pair because \(F_j^{s}\) determines the edge support on which \(F_j^{u}\) operates, while \(F_j^{u}\) determines how that state is subsequently perturbed. Thus, an individual consists of a shared design blueprint and a coupled function pair,
\begin{equation}
a_j = \big(Z_j,{F}^{\text{s}}_j,{F}^{\text{u}}_j\big).
\label{eq:blueprint-individual}
\end{equation}
Its fitness is the mean percentage relative deviation from the problem-specific reference objective over $K$ training instances:
\begin{equation}
\mathcal{F}(a_j)
=
\frac{1}{K}
\sum_{k=1}^{K}
\frac{
f_{a_j}^{(k)}-f_{\mathrm{ref}}^{(k)}
}{
f_{\mathrm{ref}}^{(k)}
}
\times 100\%.
\label{eq:fitness}
\end{equation}
Here, \(f_{a_j}^{(k)}\) is the objective value obtained by the complete generated pair and \(f_{\mathrm{ref}}^{(k)}\) is the corresponding known reference value.

\subsubsection{Theoretical Characterization of Coupled component Generation}
\label{sec:theoretical-characterization}
Let $f^{s}\in\mathcal{F}^{s}$ denote an initialization component and $f^{u}\in\mathcal{F}^{u}$ a penalty-construction component. Given a distance matrix $\mathcal{D}$, the initialization component constructs a route that is refined by a deterministic local-search operator $\mathcal{L}$:
\begin{equation}
S(f^{s})=\mathcal{L}\big(f^{s}(\mathcal{D})\big).
\label{eq:theory-initial-local-optimum}
\end{equation}
Let $A(S)\in\{0,1\}^{n\times n}$ be the undirected edge-incidence matrix of search state $S$, where
\begin{equation}
A_{ij}(S)=
\begin{cases}
1, & \text{if edge }(i,j)\text{ belongs to }S,\\
0, & \text{otherwise.}
\end{cases}
\label{eq:theory-edge-incidence}
\end{equation}
For the route-supported penalty components considered in this study, the guided distance update can be written in the unified form
\begin{equation}
\mathcal{U}_{f^{u}}(\mathcal{D},S,\mathcal{P})=\mathcal{D}+A(S)\odot B_{f^{u}}(\mathcal{D},\mathcal{P}),
\label{eq:route-supported-penalty}
\end{equation}
where $\mathcal{P}$ denotes auxiliary search-state information, including historical edge penalties or usage information, $\odot$ is the elementwise product, and $B_{f^{u}}(\mathcal{D},\mathcal{P})$ is the edge-adjustment intensity induced by $f^{u}$.

\begin{definition}[component interaction term]
\label{def:component-interaction}
For two initialization components $f^{s},\tilde f^{s}\in\mathcal{F}^{s}$ and two penalty-construction components $f^{u},\tilde f^{u}\in\mathcal{F}^{u}$, define the mixed finite difference
\begin{align}
\mathfrak{I}:={}&\mathcal{U}_{f^{u}}\big(\mathcal{D},S(f^{s}),\mathcal{P}\big)-\mathcal{U}_{f^{u}}\big(\mathcal{D},S(\tilde f^{s}),\mathcal{P}\big)\nonumber\\
&-\mathcal{U}_{\tilde f^{u}}\big(\mathcal{D},S(f^{s}),\mathcal{P}\big)+\mathcal{U}_{\tilde f^{u}}\big(\mathcal{D},S(\tilde f^{s}),\mathcal{P}\big).
\label{eq:component-interaction}
\end{align}
If $\mathfrak{I}\neq 0$, the two components have a non-separable interaction at the current search state.
\end{definition}

\begin{theorem}[State-transition non-separability of coupled heuristic components]
\label{thm:nonseparability}
Suppose that the penalty-construction operator satisfies \eqref{eq:route-supported-penalty}. If there exists an edge $(i,j)$ such that
\begin{equation}
A_{ij}\big(S(f^{s})\big)\neq A_{ij}\big(S(\tilde f^{s})\big)
\label{eq:path-difference-condition}
\end{equation}
and
\begin{equation}
\left[B_{f^{u}}(\mathcal{D},\mathcal{P})\right]_{ij}\neq\left[B_{\tilde f^{u}}(\mathcal{D},\mathcal{P})\right]_{ij},
\label{eq:penalty-difference-condition}
\end{equation}
then $\mathfrak{I}\neq 0$. Moreover, there do not exist mappings $\mathcal{G}^{s}$ and $\mathcal{G}^{u}$ that depend only on the individual initialization and penalty components, respectively, such that, for every $f^{s}\in\mathcal{F}^{s}$ and $f^{u}\in\mathcal{F}^{u}$,
\begin{equation}
\mathcal{U}_{f^{u}}\big(\mathcal{D},S(f^{s}),\mathcal{P}\big)=\mathcal{G}^{s}(f^{s};\mathcal{D},\mathcal{P})+\mathcal{G}^{u}(f^{u};\mathcal{D},\mathcal{P}).
\label{eq:additive-decomposition}
\end{equation}
\end{theorem}

Theorem~\ref{thm:nonseparability} shows that the state-transition effects between initialization and penalty components cannot be decomposed into two fixed, independently attributable contributions. The complete component pair should therefore be generated, evaluated, and selected jointly. A proof is provided in Appendix~\ref{app:proof-nonseparability}.

\begin{theorem}[Conditional mismatch elimination under shared-blueprint joint generation]
\label{thm:blueprint}
Suppose that the joint LLM generation process satisfies \eqref{eq:blueprint-distribution}--\eqref{eq:blueprint-faithfulness}. Let $(F^{s,J},F^{u,J})$ be a component pair generated under the same shared blueprint $Z$. Then
\begin{equation}
\Pr\!\left((F^{s,J},F^{u,J})\in\mathcal{C}\right)=1.
\label{eq:joint-consistency-probability}
\end{equation}
Alternatively, suppose that the initialization and penalty-construction components are generated from independent blueprints,
\begin{equation}
Z_s,Z_u\overset{\mathrm{i.i.d.}}{\sim}\pi_{\theta}(\cdot\mid x),
\label{eq:independent-blueprints}
\end{equation}
with
\begin{equation}
F^{s,I}\sim q_{\theta}^{s}(\cdot\mid x,Z_s),
\qquad
F^{u,I}\sim q_{\theta}^{u}(\cdot\mid x,Z_u).
\label{eq:independent-generation}
\end{equation}
Then
\begin{equation}
\Pr\!\left((F^{s,I},F^{u,I})\in\mathcal{C}\right)=\sum_{z\in\mathcal{Z}}\pi_{\theta}(z\mid x)^2,
\label{eq:independent-consistency-probability}
\end{equation}
and the blueprint-mismatch probability is
\begin{equation}
\Pr\!\left((F^{s,I},F^{u,I})\notin\mathcal{C}\right)=1-\sum_{z\in\mathcal{Z}}\pi_{\theta}(z\mid x)^2.
\label{eq:mismatch-probability}
\end{equation}
If at least two blueprints have positive probability under $\pi_{\theta}(\cdot\mid x)$, then the mismatch probability in \eqref{eq:mismatch-probability} is strictly positive.
\end{theorem}

Theorem~\ref{thm:blueprint} characterizes the structural advantage of joint LLM generation in that an explicit shared blueprint guarantees design-principle consistency within a generated pair, whereas independently generated components can be mismatched with positive probability. The proof is provided in Appendix~\ref{app:proof-blueprint}.

\subsubsection{Framework-Level Search Enhancement}
Relative to the classical TSP GLS in Algorithm~\ref{alg:gls}, the online procedures in Algorithms~\ref{alg:LLMgls} and~\ref{alg:LLMgls_cvrp}, all reported in Appendix~\ref{app:routing-gls-frameworks}, introduce four fixed framework-level mechanisms: (i) \textbf{influential-edge selection}, which identifies the top-5 positive entries of \(\Delta=\max(\mathcal{D}'-\mathcal{D},0)\) as guided perturbation targets (Lines~7--9); (ii) \textbf{endpoint-targeted perturbation}, which applies local moves around the selected edge endpoints, using 2-opt and relocate for TSP plus inter-route swap for CVRP (Lines~10--20, Lines~10--24). Here, \(\delta\) denotes the objective change of a candidate local move, with \(\delta<0\) indicating accepted improvement; (iii) \textbf{post-perturbation local improvement}, which refines the perturbed solution after guidance from \(\mathcal{D}'\) while evaluating quality under \(\mathcal{D}\) (Lines~22--26, Lines~26--30); (iv) \textbf{periodic incumbent restoration}, which resets the current state to the best-so-far solution every 50 outer iterations to control search drift (Lines~27--29, Lines~31--33). These mechanisms are fixed components of the online search framework rather than generated by the LLM.

\section{Experiments}
\label{Sec:Experiments}
In this section, we evaluate the proposed LLM-HCJG framework and the resulting LLM-enhanced GLS algorithm. The entire process is implemented in Python and executed on an Intel Core i9-14900HX (2.20,GHz). To ensure closer alignment with the compared LLM-based methods, we follow ReEvo’s final-elite selection protocol: three independent full-evolution runs (with different random seeds) produce three run-specific elites, which are evaluated on the released public TSP/CVRP validation sets. Only the best-validated elite is subsequently tested 10 times on each test instance, with the mean performance reported. All tests are terminated when either the prescribed time budget is exhausted or the maximum number of 1,000 outer iterations is reached, with limits of 100 seconds for TSP and 20 seconds for CVRP. The experimental setup is presented first, followed by an analysis of the training process and the best-performing algorithm discovered by LLM-HCJG. The resulting algorithm is subsequently compared against representative baselines on synthetic instances and public benchmark suites. Finally, ablation studies on joint combinatorial dependency coevolution are performed to assess the role of cross-component co-adaptation in the proposed framework. All results and source codes will be publicly available at github website upon acceptance of this paper.
\subsection{Experiments Settings and Evolutionary Outcome}
\label{SubSec:Dataset}
This subsection reports the results of the LLM-HCJG training phase and its complete hyperparameter configuration in Table~\ref{tab:algo-settings}. We adopt the common LLM training budgets used in previous studies \citep{wu2025efficient, ye2024reevo}: 10 generations of coevolution $N_g$ with a population size $N$ of 10. Parameters $\sigma_1$, $\sigma_2$, $l$, and $s$ follow the EoH settings \citep{9e7eceac6e31432aafced5e02a49de16, liu2024example} to balance activation intensity, generation granularity, and search diversity. Candidate fitness is estimated on $K=5$ instances sampled from a pool of 64, providing a trade-off between evaluation reliability and computational overhead. Each instance contains 100 customer nodes uniformly sampled in the unit square $[0,1] \times [0,1]$, together with LKH reference solutions. CVRP instances additionally include integer customer demands in $[1,9]$, a fixed vehicle capacity of 50, with an average of 10--11 vehicles.

Under GPT-4o-mini, the LLM-HCJG training phase averagely requires 34 minutes for TSP and 183 minutes for CVRP. For each task, the training records contain 410 candidate-generation attempts.TSP and CVRP produce 361 and 390 valid candidates, corresponding to effective generation rates of 88.0\% and 95.12\%.

Figure~\ref{fig:gap} illustrates the behavior of the best-validated elite heuristic gap over generations. The best individuals of TSP and CVRP emerge at generations 4 and 7, respectively, indicating effective convergence within the adopted training budget. The average LLM training time for the TSP-GLS task is much lower than EoH \citep{wu2025efficient, ye2025large} required nearly two days of search and approximately 2,000 LLM queries across 64 TSP100 training instances \citep{9e7eceac6e31432aafced5e02a49de16, liu2024example}. This highlights that LLM-HCJG can achieve competitive heuristic discovery with a relatively modest training budget.
\begin{figure}[t]
\centering
\begin{subfigure}{0.496\textwidth}
    \centering
    \includegraphics[width=\textwidth]{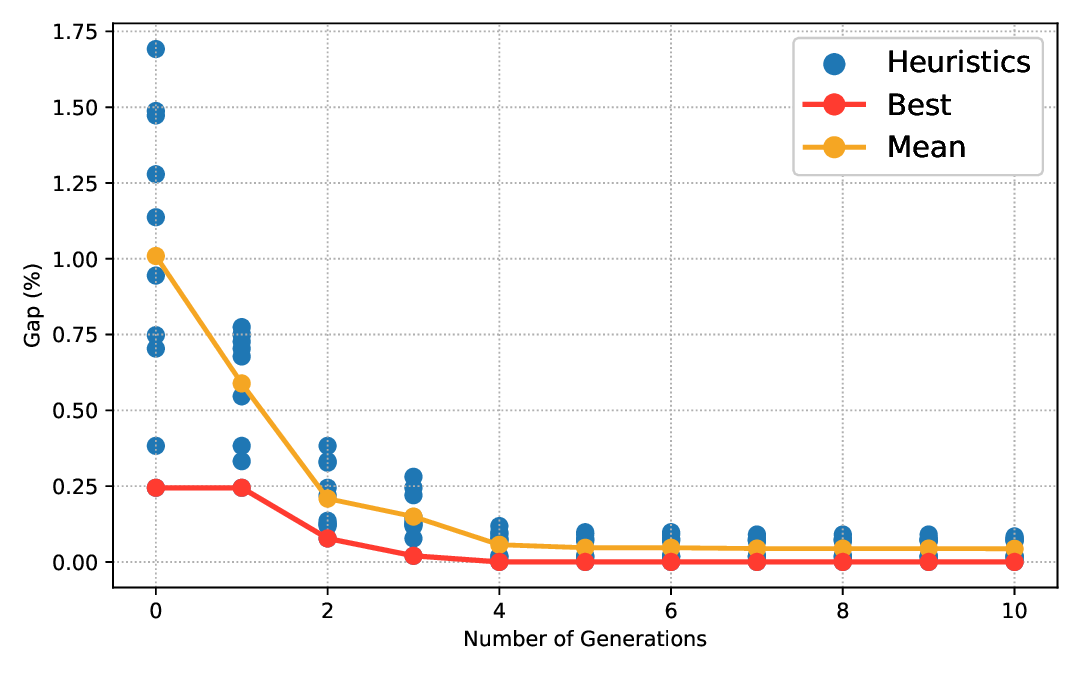}
    \caption{TSP}
    \label{fig:gap_tsp}
\end{subfigure}
\hfill
\begin{subfigure}{0.496\textwidth}
    \centering
    \includegraphics[width=\textwidth]{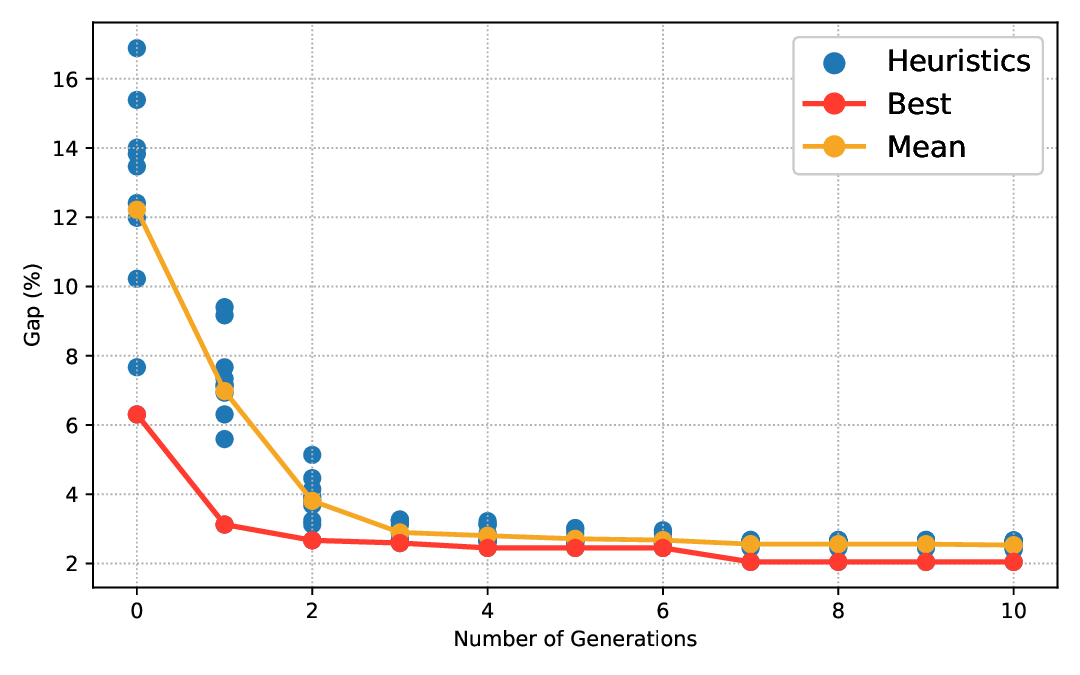}
    \caption{CVRP}
    \label{fig:gap_cvrp}
\end{subfigure}
\caption{Convergence of the Heuristic Gap under the LLM-HCJG Framework}
\label{fig:gap}
\end{figure}

Figure~\ref{fig:evolution_tree} presents the evolutionary trajectory plots of the final populations. Each node represents a candidate heuristic labeled by its index and fitness value. Orange, blue, and red denote the initial, final-retained, and best-heuristic individuals, respectively. The different-colored directed edges trace the evolution of the parent to the offspring and distinguish the corresponding evolutionary operators. In TSP and CVRP, the best individuals are No.~151 and No.~251, generated via the e1 and m1 operators. Their final performance gaps settle at 2.046\% and 0.000\%, with both preserved in subsequent generations.

Details on the best-performing TSP and CVRP elite heuristics code with mathematical interpretation are provided in Appendices~\ref{apx:tsp-elite} and~\ref{apx:cvrp-elite}, respectively. For TSP, the two functions follow a stochastic heuristic design that combines deterministic structural information with random factors $\xi_j$ and $\eta_{ij}$. For CVRP, both functions follow a structure-aware penalty-augmented distance shaping principle rather than relying on raw distances.

\subsubsection{Structural Alignment of the Best-Evolved component Pairs}
\label{sec:structural-indicators}
To assess structural interaction in the best-performing component pairs, two complementary post-hoc indicators are adopted:
\[
Q_h(a_j;x)
\in
\left\{
\mathrm{Focus}_{h}(a_j;x),
\mathrm{Diversity}_{h}(a_j;x)
\right\}.
\]
At the \(h\)-th penalty-update call on instance \(x\), \(\mathrm{Focus}_{h}(a_j;x)\) measures how strongly the gap signals align with the current search state initialized by \(F_j^{s}\) and refined by the subsequent search process, whereas \(\mathrm{Diversity}_{h}(a_j;x)\) measures whether the signals are differentiated across edges rather than uniformly distributed. The final value is averaged over \(H\) penalty-update calls for each of the \(M\) evaluated instances:
\begin{equation}
Q(a_j)
=
\frac{1}{MH}
\sum_{m=1}^{M}
\sum_{h=1}^{H}
Q_h(a_j;x_m).
\label{eq:structural-aggregation}
\end{equation}
Appendix~\ref{apx:structural-indicators} shows complete definitions of these indicators. For the best TSP and CVRP individuals, the measured values are \((\mathrm{Focus},\mathrm{Diversity})=(0.980,0.318)\) and \((0.978,0.438)\). These results indicate that the jointly evolved component pairs realize a state-aligned and selective perturbation pattern.
\begin{figure}[t]
\centering
\begin{subfigure}{0.496\textwidth}
    \centering
    \includegraphics[width=\textwidth]{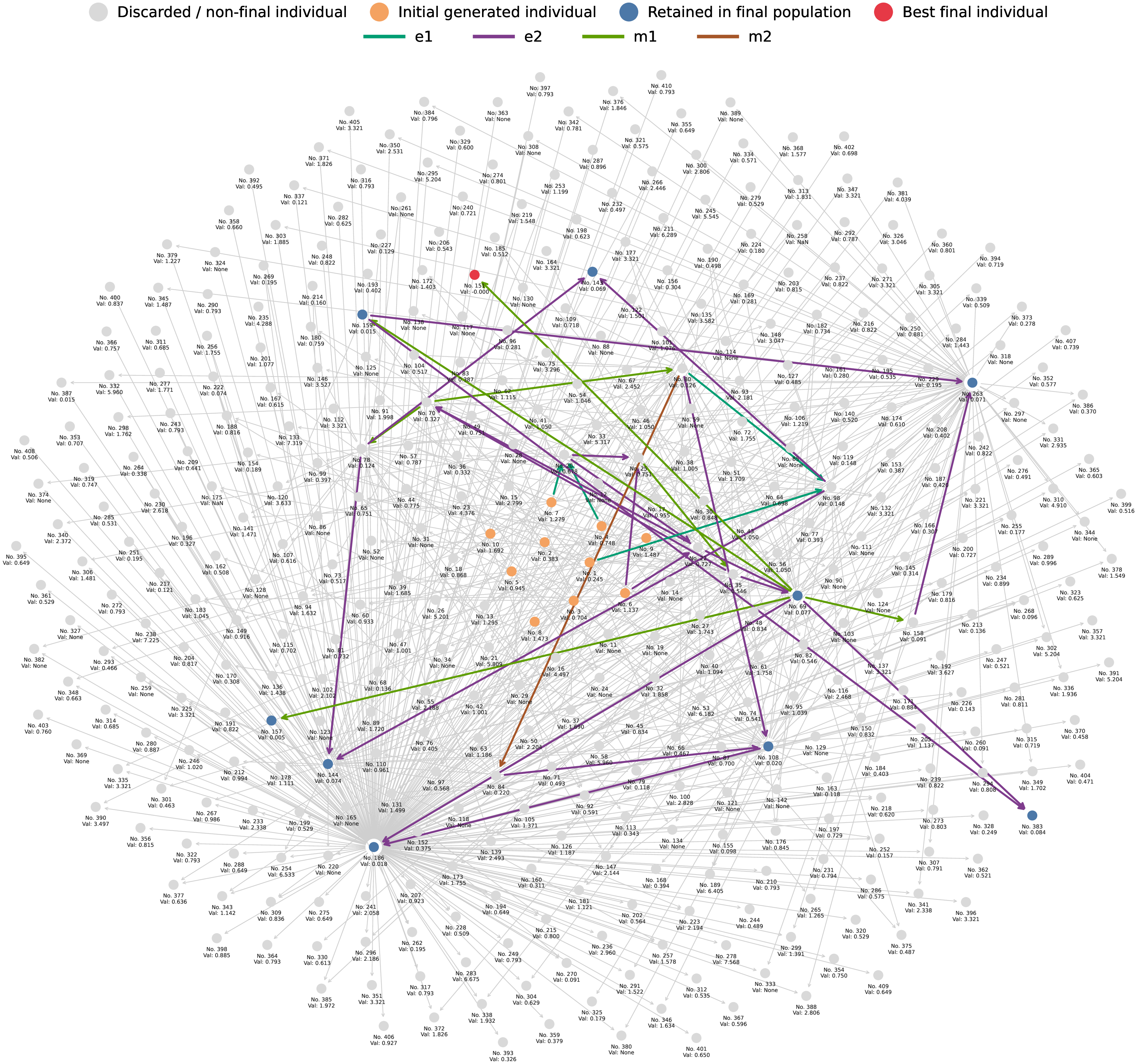}
    \caption{TSP}
    \label{fig:tree_tsp}
\end{subfigure}
\hfill
\begin{subfigure}{0.496\textwidth}
    \centering
    \includegraphics[width=\textwidth]{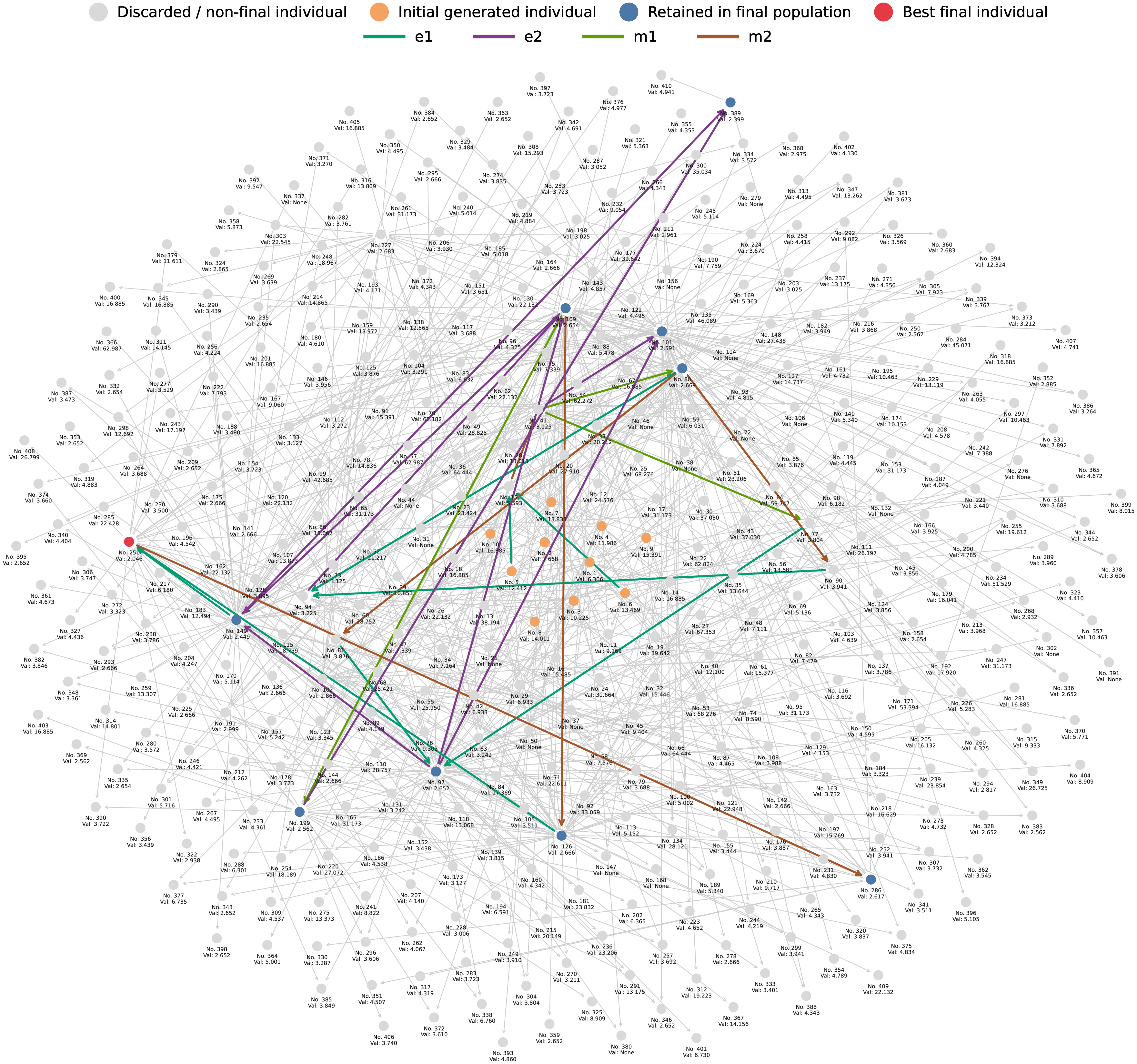}
    \caption{CVRP}
    \label{fig:tree_cvrp}
\end{subfigure}
\caption{Final-population evolutionary trajectories for TSP and CVRP}
\label{fig:evolution_tree}
\end{figure}
\begin{table}[!hbtp]
  \centering
  \caption{Parameter Settings of LLM-HCJG}
  \label{tab:algo-settings}
  \begin{threeparttable}
    \begin{tabular}{lrr}
      \toprule
      Parameter & Meaning & Value \\
      \midrule
      $\sigma_{1}$ & Probability for crossover  & 1.0 \\
      $\sigma_{2}$ & Probability for mutation  & 0.5 \\
      LLM & Chosen large language model & GPT-4o-mini \\
      $N$ & Population size & 10 \\
      $N_g$ & Number of generations & 10 \\
      $l$ & Number of parent individuals & 2 \\
      $s$ & Number of offspring individuals  & 1 \\
      $K$ & Number of training instances & 5 \\
      \bottomrule
    \end{tabular}
  \end{threeparttable}
\end{table}
\subsection{Baseline Methods}
\label{SubSec:Comparison with Baseline Methods}
The proposed method is evaluated against representative baselines from three categories:
\begin{itemize}[leftmargin=*]
  \item \textbf{Handcrafted heuristic approaches:} including local search (LS) \citep{croes1958method}, guided local search (GLS) \citep{voudouris1999guided}, knowledge-GLS (KGLS) \citep{arnold2019knowledge} and elite biased GLS (EBGLS) \citep{shi2018eb};

  \item \textbf{Neural network-based approaches:} including attention model (AM) \citep{kool2018attention}, policy optimization with multiple optima (POMO) \citep{kwon2020pomo}, large-scale enhanced heavy decoder (LEHD) \citep{luo2023neural}, graph neural network GLS (GNNGLS) \citep{hudson2021graph} and Neural GLS (NeuralGLS) \citep{sui2024neuralgls};

  \item \textbf{Prior LLM-EC methods:} EoH \citep{9e7eceac6e31432aafced5e02a49de16} and ReEvo \citep{ye2024reevo}.
\end{itemize}

For TSP, comparisons are conducted across all method categories considered. For CVRP, the lack of directly comparable neural results and canonical CVRP variants of KGLS/EBGLS restricts the comparison only to GLS, EoH, and ReEvo. The results for neural network-based approaches are taken from the original publications, given that the substantial method-specific training protocols and computational resources are beyond the scope of this study. EoH and ReEvo are only generated for the \texttt{LLM\_update\_edge\_distance} function, while \texttt{LLM\_select\_next\_node} is fixed to the default nearest-feasible-customer rule. For TSP, the best release from \citep{9e7eceac6e31432aafced5e02a49de16} is adopted directly, as it conforms to the same dynamic edge-distance update interface and avoids unnecessary retraining variability. In contrast, ReEvo’s published heuristic specific to classical GLS is incompatible. The best individuals of ReEvo for TSP and EoH/ReEvo for CVRP, listed in Appendix~\ref{apx:eoh-reevo-best-penalty}, are re-evolved under the same training setting as LLM-HCJG.
\subsection{Performance Comparison on Synthetic Instances}
\label{Experiment 1}
Synthetic benchmarks are generated as random Euclidean instances in $[0,1]^2$, comprising 1{,}000 TSP instances for each size in $\{20,50,100\}$ and 64 CVRP instances for each customer size in $\{20,50,100\}$. Extraly for CVRP, the depot is fixed at $(0.5, 0.5)$, customer demands are uniformly sampled from integers in $[1, 9]$, and the vehicle capacity is set to 50. The optimality gaps compare the objective value $f^{(i)}$ obtained on instance $i$ in each independent run against the corresponding LKH3-generated reference value $f_{\mathrm{ref}}^{(i)}$ \citep{helsgaun2017extension}. The mean percentage gap is defined as
\begin{equation}
\operatorname{Gap}
=
\frac{1}{N_{\mathrm{test}}}
\sum_{i=1}^{N_{\mathrm{test}}}
\frac{
f^{(i)}-f_{\mathrm{ref}}^{(i)}
}{
f_{\mathrm{ref}}^{(i)}
}
\times 100\%,
\label{eq:test_gap}
\end{equation}
where $N_{\mathrm{test}}=1000$ for TSP and $N_{\mathrm{test}}=64$ for CVRP at each problem size. 

To distinguish framework-level gains from LLM-designed heuristic gains, Framework-Augmented GLS (FA-GLS) augments the classical GLS core with the search mechanisms of LLM-HCJG, while preserving greedy initialization, edge selection by \(d_{ij}/(1+p_{ij})\), and augmented distance construction \(D'=D+\lambda P\). The results are  reported in Tables~\ref{tab:tsp_synthetic} and~\ref{tab:cvrp_synthetic}.
\begin{table}[H]
  \centering
  \caption{Performance comparison on synthetic TSP instances}
  \label{tab:tsp_synthetic}
  \footnotesize
  \begin{threeparttable}
    \setlength{\tabcolsep}{3pt}
    \begin{tabular}{lcccccc}
      \toprule
      \multirow{2}{*}{Method} &
      \multicolumn{2}{c}{\textbf{TSP20}} &
      \multicolumn{2}{c}{\textbf{TSP50}} &
      \multicolumn{2}{c}{\textbf{TSP100}} \\
      \cmidrule(lr){2-3}\cmidrule(lr){4-5}\cmidrule(lr){6-7}
      & Gap (\%) & Time (s) & Gap (\%) & Time (s) & Gap (\%) & Time (s) \\
      \midrule
      AM           & 0.069 & 0.038 & 0.494 & 0.124 & 2.368 & 0.356 \\
      POMO         & 0.120 & /     & 0.640 & /     & 1.070 & /     \\
      LEHD         & 0.950 & /     & 0.485 & /     & 0.577 & /     \\
      GNNGLS       & \textbf{0.000} & 10.010 & 0.009 & 10.037 & 0.698 & 10.108 \\
      NeuralGLS    & \textbf{0.000} & 10.005 & 0.003 & 10.011 & 0.470 & 10.024 \\
      \midrule
      LS           & 0.918 & 1.003 & 2.869 & 1.006 & 3.708 & 1.029 \\
      GLS          & \textbf{0.000} & 1.050 & 0.194 & 1.140 & 1.212 & 1.414 \\
      FA-GLS & \textbf{0.000} & 1.138 & \textbf{0.000} & 1.289 & 0.140 & 1.975 \\
      KGLS-r       & \textbf{0.000} & 1.254 & \textbf{0.000} & 1.619 & 0.168 & 2.382 \\
      KGLS-c       & \textbf{0.000} & 1.653 & \textbf{0.000} & 1.409 & 0.178 & 2.001 \\
      EBGLS        & \textbf{0.000} & 1.053 & 0.007 & 1.164 & 0.246 & 1.484 \\
      \midrule
      EoH          & \textbf{0.000} & 4.654 & \textbf{0.000} & 7.865 & 0.147 & 21.228 \\
      ReEvo        & \textbf{0.000} & 2.347 & \textbf{0.000} & 3.157 & 0.293 & 7.683 \\
      LLM-HCJG     & \textbf{0.000} & 1.757 & \textbf{0.000} & 2.960 & \textbf{0.070} & 5.527 \\
      \bottomrule
    \end{tabular}
  \end{threeparttable}
\end{table}

\begin{table}[H]
  \centering
  \caption{Performance comparison on synthetic CVRP instances}
  \label{tab:cvrp_synthetic}
  \footnotesize
  \begin{threeparttable}
    \setlength{\tabcolsep}{3pt}
    \begin{tabular}{lcccccc}
      \toprule
      \multirow{2}{*}{Method} &
      \multicolumn{2}{c}{\textbf{CVRP20}} &
      \multicolumn{2}{c}{\textbf{CVRP50}} &
      \multicolumn{2}{c}{\textbf{CVRP100}} \\
      \cmidrule(lr){2-3}\cmidrule(lr){4-5}\cmidrule(lr){6-7}
      & Gap (\%) & Time (s) & Gap (\%) & Time (s) & Gap (\%) & Time (s) \\
      \midrule
      GLS           & 1.052 & 1.416 & 5.494 & 5.711 & 8.362 & 17.162 \\
      FA-GLS & 0.860 & 3.908 & 4.592 & 17.728 & 4.523 & 18.797 \\
      ReEvo         & 1.174 & 5.392 & 5.782 & 18.293 & 11.276 & 20.577 \\
      EoH           & 2.018 & 3.992 & 8.008 & 13.619 & 10.852 & 18.961 \\
      LLM-HCJG      & \textbf{0.799} & 1.305 & \textbf{0.927} & 10.315 & \textbf{3.526} & 20.687 \\
      \bottomrule
    \end{tabular}
  \end{threeparttable}
\end{table}

Across both problems, FA-GLS improves upon the original GLS, reducing the large-instance gaps from 1.212\% to 0.140\% on TSP100 and from 8.362\% to 4.523\% on CVRP100, thereby confirming the contribution of the added search mechanisms. On top of this strengthened framework, LLM-HCJG achieves the best overall solution quality among the evaluated baselines, attaining gaps of 0.070\% on TSP100 and 3.526\% on CVRP100. Compared with the prior LLM-enhanced baselines, LLM-HCJG reduces the TSP100 gap from 0.147\% for EoH and 0.293\% for ReEvo to 0.070\%, and the CVRP100 gap from 10.852\% and 11.276\% to 3.526\%, respectively. This consistent advantage suggests that jointly adapting the construction and guidance components provides a more effective search bias than optimizing the edge-distance update rule alone.
\subsection{Performance Comparison on TSPLIB and CVRPLIB}
\label{Experiment 2}
To further assess cross-instance transfer, experiments are conducted on public benchmark instances from TSPLIB and CVRPLIB. The 29 TSPLIB instances follow the selection from the original EoH study. 12 representative instances are chosen from the A, B, E, P, M, and X series, spanning scales from small to near-200, sparse to dense fleet settings, and both capacity-tight and long-route structures. Results denote the performance gaps (\%) to the best-known reference solutions. Number denotes the count of instances achieving the best or tied-best result.

Tables~\ref{apx:tsplib-results} and~\ref{apx:cvrplib-results} report the detailed benchmark results. On TSPLIB, LLM-HCJG achieves the lowest average gap of 0.017\% and the best or tied-best performance on 28 of 29 instances. FA-GLS already reduces the average gap of GLS from 1.119\% to 0.140\%, but LLM-HCJG further improves it to 0.017\%, outperforming EoH (3.407\%) and ReEvo (0.030\%). The CVRPLIB results show an even clearer separation. LLM-HCJG obtains the best or tied-best result on all 12 instances and reduces the average gap to 1.686\%, compared with GLS (6.240\% ), FA-GLS (6.152\%), EoH (3.601\%), and ReEvo (9.451\%). Unlike TSPLIB, CVRPLIB exhibits stronger constraint-induced sensitivity, where FA-GLS yields less stable gains and may deteriorate on capacity-tight or structurally distinct instances. These results indicate that joint component evolution improves not only the underlying search framework but also the cross-instance generalization of LLM-designed heuristics on heterogeneous structures.

The consistent advantage of LLM-HCJG is particularly useful for constrained multi-route problems. Representative route-level comparisons for \texttt{E-n76-k10} and \texttt{P-n55-k15} are shown in Figures~\ref{fig:e-n76-k10-route} and~\ref{fig:p-n55-k15-route}, respectively. In several baseline solutions, especially FA-GLS and ReEvo, routes tend to connect customers across spatially separated regions, producing long inter-region arcs and unnecessary returns around the depot, which directly increases the total travel cost. For the LLM-enhanced methods, \texttt{LLM\_update\_edge\_distance} introduces diversification but does not explicitly encode route compactness, cross-region separation, or the coupling between construction quality and feasible route-level perturbation. Moreover, \texttt{LLM\_select\_next\_node} is fixed to the nearest-feasible-customer rule, leaving structural biases in the initial route partition only partially repairable by later penalty-guided search. In contrast, LLM-HCJG jointly evolves two search direction components, aligning initial route formation with subsequent edge penalization. Its advantage thus comes not merely from stronger perturbation, but from more accurate localization of where perturbation should be applied, helping preserve geographically coherent customer clusters and guide endpoint-based search toward edges that better correct route-level structure.
\begin{figure}[!b]
  \centering
  \includegraphics[width=\textwidth]{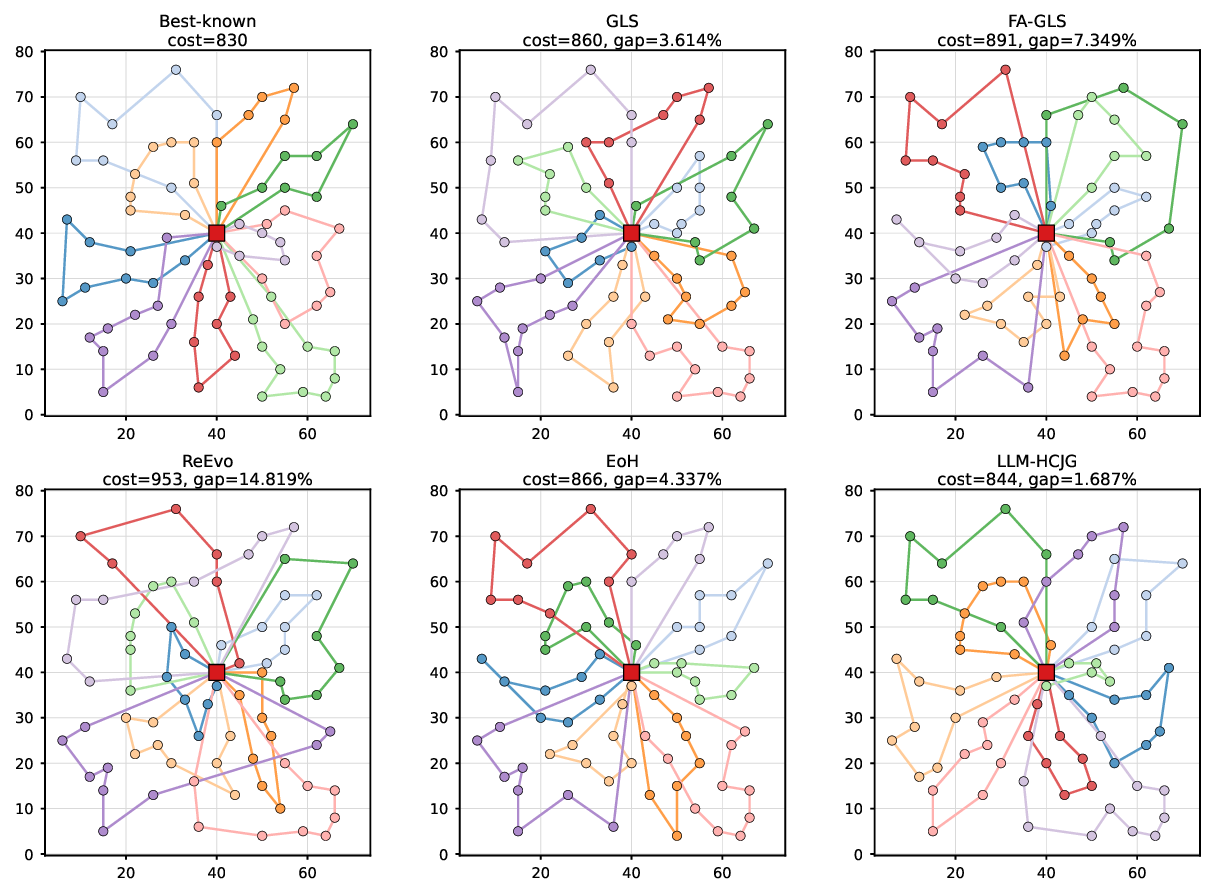}
  \caption{Route-level comparison on \texttt{E-n76-k10}.}
  \label{fig:e-n76-k10-route}
\end{figure}

\begin{figure}[!t]
  \centering
  \includegraphics[width=\textwidth]{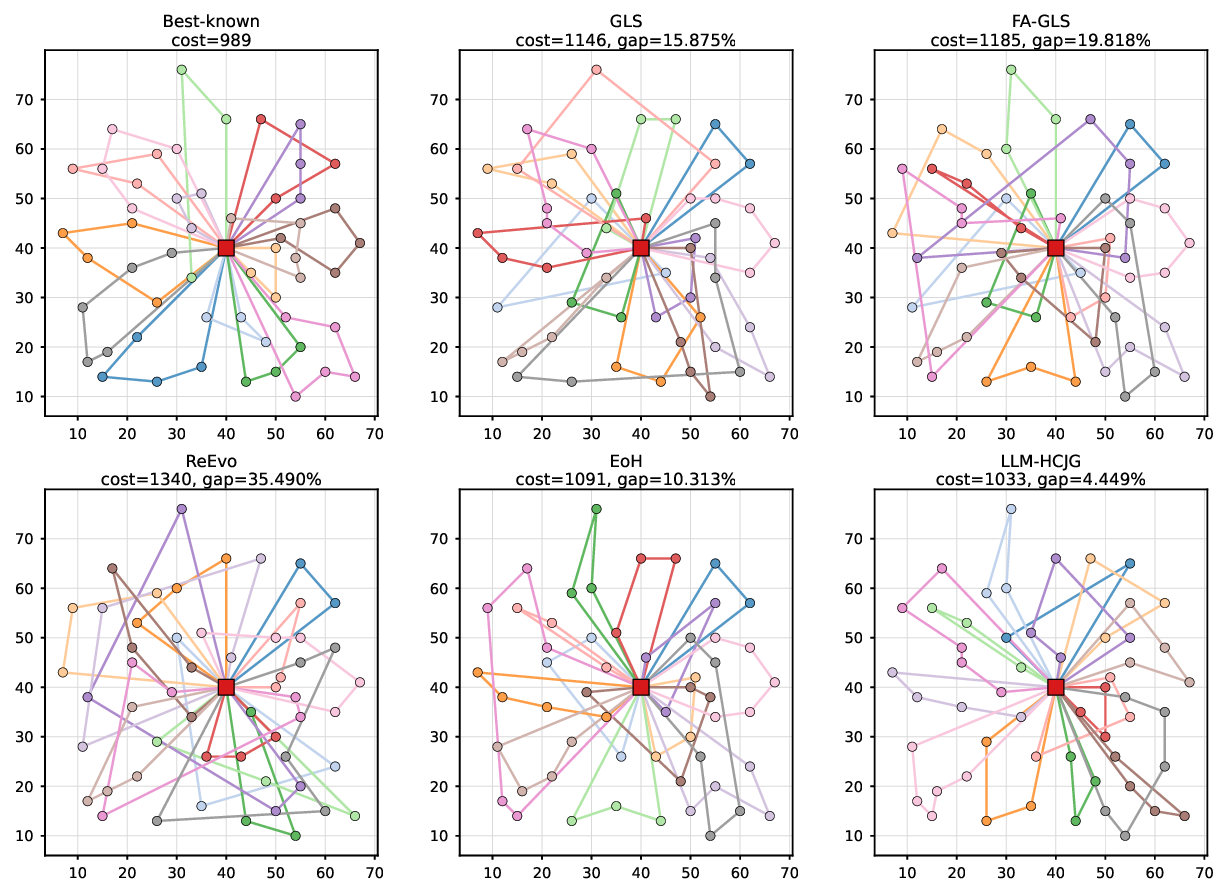}
  \caption{Route-level comparison on \texttt{P-n55-k15}.}
  \label{fig:p-n55-k15-route}
\end{figure}

\subsection{Ablation Studies on Joint Combinatorial Dependency coevolution}
\label{SubSec:Ablation Joint Coevolution}

Ablation studies on joint combinatorial dependency coevolution are conducted on TSP, where publicly available EoH studies provide a reliable reference for recombination~\citep{liu2023algorithm,9e7eceac6e31432aafced5e02a49de16}. Four component pairings are evaluated under the same algorithmic framework as Appendix~\ref{apx:llmgls-tsp}, differing only in the sources of the initialization and penalty-construction modules: EoH--EoH, EoH--HCJG, HCJG--EoH, and HCJG--HCJG. Detailed results appear in Appendix~\ref{app:joint-coevolution-ablation}.

Tables~\ref{tab:ablation_synthetic_tsp} and~\ref{tab:ablation_tsplib},
together with Figure~\ref{fig:ablation-tsplib-nonidentical-gaps}, show that HCJG--HCJG achieves the strongest overall performance. On synthetic instances, it attains the best or tied-best gap at all scales and reduces the TSP100 gap to 0.070\%, compared with 0.082\%, 0.152\%, and 0.159\% for the other combinations. On TSPLIB, it also obtains the lowest mean gap of 0.017\%, with EoH--HCJG (0.030\%), HCJG--EoH (0.611\%), and EoH--EoH (2.788\%) yielding larger gaps.
The asymmetric degradation of the two hybrid variants indicates that the intact jointly evolved pair benefits from cross-component compatibility rather than simple recombination of independently designed modules. The results above support Theorems~\ref{thm:nonseparability} and~\ref{thm:blueprint}, indicating both the nonzero interaction \(\mathfrak{I}\) and a higher probability of falling in the consistency set \(\mathcal{C}\).

The paired analyses in Appendices~\ref{apx:stat-public} and~\ref{apx:stat-ablation} cover the public benchmarks and cross-component ablation. All bootstrap 95\% confidence intervals lie above zero, supporting stable performance advantages under instance-level resampling, while all two-sided Wilcoxon signed-rank tests satisfy $p<0.05$, confirming statistically significant paired differences.
\section{Conclusion}
\label{Sec:Conclusion}
This paper develops LLM-HCJG, a joint-generation framework that advances LLM-enhanced evolutionary computation from isolated heuristic construction to coordinated component coevolution. By jointly evolving the solution-initialization and penalty-construction components under a shared design blueprint, LLM-HCJG reshapes the search direction of GLS while preserving its interpretable backbone across both TSP and CVRP. Theoretical and ablation analyses confirm their non-separability and effectiveness in terms of pairwise compatibility. Experiments on TSP/CVRP synthetic instances and public benchmarks demonstrate strong solution quality, efficient training, and robust generalization, consistently surpassing single-component LLM-enhanced baselines. These findings suggest that LLMs can move beyond heuristic rule generation toward shaping the structural search bias of concrete algorithmic frameworks.

Future research will focus on three improvements: (i) scalable deployment on larger and more heterogeneous instances, (ii) tighter integration of joint heuristic evolution with model fine-tuning for higher-quality generation, and (iii) extension of joint interacting-component generation to richer routing environments and other metaheuristic frameworks.

\bibliographystyle{apalike}
\spacingset{1}
\bibliography{IISE-Trans}

\clearpage
\appendix
\section{Prompt Blocks Utilized within the LLM-HCJG}
\label{app:tsp-prompt}

\subsection{Basic Attribute Description for TSP}
\label{apx:tsp-basic-attributes}

\begin{tcolorbox}[
    title={Basic Attribute Description For TSP},
    width=\textwidth,
    breakable
]

\textbf{\color{brown}self.problem\_description}\par\noindent
{\ttfamily
Travelling Salesman Problem (TSP): given a set of cities and the
pairwise distances between them, find a shortest Hamiltonian cycle
that visits each city exactly once and returns to the starting city.
}

\vspace{10pt}
\hrule
\vspace{10pt}

\textbf{\large Solution Initialization
(\texttt{select\_next\_node})}

\vspace{6pt}
\textbf{\color{brown}self.prompt\_task1}\par\noindent
{\ttfamily
Given a set of nodes with their coordinates, please help me design a strategy to find the shortest route that visits each node once and returns to the starting node. The task can be solved step-by-step by starting from the current node and iteratively choosing the next node. Help me design a novel heuristic that is different from the literature to select the next node in each step.
}

\vspace{6pt}
\textbf{\color{brown}self.prompt\_func\_name1}\par\noindent
\texttt{select\_next\_node}

\vspace{6pt}
\textbf{\color{brown}self.prompt\_func\_inputs1}\par\noindent
\texttt{current\_node},
\texttt{destination\_node},
\texttt{unvisited\_nodes},
\texttt{distance\_matrix}

\vspace{6pt}
\textbf{\color{brown}self.prompt\_func\_outputs1}\par\noindent
\texttt{next\_node}

\vspace{6pt}
\textbf{\color{brown}self.prompt\_inout\_inf1}\par\noindent
{\ttfamily
The function signature must exactly be
def select\_next\_node. 'current\_node', 'destination\_node', 'next\_node', and
'unvisited\_nodes' (one-dimensional array) are node IDs which must be integers. 'distance\_matrix' (two-dimensional array with shape (100,100)) is the distance matrix of nodes. The function must return exactly one integer node ID selected from unvisited\_nodes, not an index position, score, array, list, tuple, None, NaN, or infinity.
}

\vspace{10pt}
\hrule
\vspace{10pt}

\textbf{\large Penalty Construction
(\texttt{update\_edge\_distance})}

\vspace{6pt}
\textbf{\color{brown}self.prompt\_task2}\par\noindent
{\ttfamily
Given an edge distance matrix, a local optimal tour, and edge usage counts, design a signed global edge-guidance strategy for TSP. The returned matrix should express how bad or good it is to include each
edge in a solution: Here gap means updated\_edge\_distance -
edge\_distance: positive gap means the edge should be selectively penalized or avoided, with rank-discriminative signals sufficiently large to affect top-k perturbation while remaining finite and bounded relative to edge distances, negative gap means the edge should be encouraged or preserved, and zero gap means no guidance. 
}

\vspace{6pt}
\textbf{\color{brown}self.prompt\_func\_name2}\par\noindent
\texttt{update\_edge\_distance}

\vspace{6pt}
\textbf{\color{brown}self.prompt\_func\_inputs2}\par\noindent
\texttt{edge\_distance},
\texttt{local\_opt\_tour},
\texttt{edge\_n\_used}

\vspace{6pt}
\textbf{\color{brown}self.prompt\_func\_outputs2}\par\noindent
\texttt{updated\_edge\_distance}

\vspace{6pt}
\textbf{\color{brown}self.prompt\_inout\_inf2}\par\noindent
{\ttfamily
The function signature must exactly be def update\_edge\_distance. 'local\_opt\_tour' is a one-dimensional integer array with shape (100,), ordered from node 0's successor to node 0 without explicit closure. 'edge\_distance' (two-dimensional array with shape (100,100)) and 'edge\_n\_used' (two-dimensional array with shape (100,100)) are matrices. 'edge\_n\_used' includes the number of times each edge is used during perturbation and may contain zeros. The function must return updated\_edge\_distance as a numeric NumPy-compatible matrix with exactly the same shape as edge\_distance, not a scalar, list, tuple, None, NaN, or infinity. Do not modify edge\_n\_used in-place. The returned matrix should be symmetric, finite, and non-negative. Avoid numerically explosive formulas such as 1/(edge\_n\_used + 1e-6), edge\_n\_used**2.
}

\vspace{10pt}
\hrule
\vspace{10pt}

\textbf{\color{brown}self.prompt\_other\_inf}\par\noindent
{\ttfamily
The node indices range from 0 to 99. Import the necessary Python components. When calling any function, match the number, names and types of arguments to the function signature. Don't miss the 'return' at the end. Avoid None, NaN or infinite. All are NumPy arrays by using NumPy methods or convert to other types before applying non-NumPy method. In order to avoid divide-by-zero, add a small epsilon (e.g.\ +1e-6) to denominators every dividing. Do not call any helper function or use any auxiliary variable unless it is fully defined inside the same function. Treat node IDs and array positions separately. Never index a candidate-length array by node ID or apply element-wise operations to arrays with unequal lengths. Each function must use only its own input arguments and locally defined variables; do not use variables from the other function. Never use edge\_n\_used in select\_next\_node; edge\_n\_used may only be used in update\_edge\_distance. If a quantity is not in the function signature, do not use it unless you can compute it directly from the provided inputs inside the same function.
}
\end{tcolorbox}

\subsection{Basic Attribute Description for CVRP}
\label{apx:cvrp-basic-attributes}

\begin{tcolorbox}[
    title={Basic Attribute Description For CVRP},
    width=\textwidth,
    breakable
]

\textbf{\color{brown}self.problem\_description}\par\noindent
{\ttfamily
Capacitated Vehicle Routing Problem (CVRP): given a depot, a set of customers with known demands, pairwise travel distances, and identical vehicles with a fixed capacity, construct feasible routes that each start and end at the depot, serve every customer exactly once without exceeding vehicle capacity, and use exactly the prescribed number K of non-empty vehicle routes while minimizing the total routing cost.
}

\vspace{10pt}
\hrule
\vspace{10pt}

\textbf{\large Solution Initialization
(\texttt{select\_next\_node})}

\vspace{6pt}
\textbf{\color{brown}self.prompt\_task1}\par\noindent
{\ttfamily
Given the current vehicle node, the feasible customers that can be served without violating the capacity constraint, the remaining unserved customers, the remaining vehicle capacity, the customer demands, and the distance matrix, please help me design a strategy to select the next customer during CVRP route construction. The task can be solved step-by-step by starting from the current vehicle node and iteratively choosing the next feasible customer until all customers are served. The goal is to build a high-quality feasible initial solution. Help me design a novel algorithm that is different from the algorithms in literature.
}

\vspace{6pt}
\textbf{\color{brown}self.prompt\_func\_name1}\par\noindent
\texttt{select\_next\_node}

\vspace{6pt}
\textbf{\color{brown}self.prompt\_func\_inputs1}\par\noindent
{\raggedright
\texttt{current\_node},
\texttt{feasible\_customers},
\texttt{remaining\_customers},
\texttt{remaining\_capacity},
\texttt{demands},
\texttt{distance\_matrix}
\par}

\vspace{6pt}
\textbf{\color{brown}self.prompt\_func\_outputs1}\par\noindent
\texttt{next\_node}

\vspace{6pt}
\textbf{\color{brown}self.prompt\_inout\_inf1}\par\noindent
{\ttfamily
The function signature must exactly be def select\_next\_node. 'current\_node' and returned 'next\_node' are node IDs and must be integers. 'feasible\_customers' and 'remaining\_customers' are one-dimensional arrays of customer IDs and do not include the depot. 'feasible\_customers' is a subset of 'remaining\_customers'. 'remaining\_capacity' is a non-negative integer.  demands' is a one-dimensional array and demands[0] = 0 for the depot. 'distance\_matrix' is a symmetric two-dimensional array with shape (n, n). When this function is called, 'feasible\_customers' is guaranteed to be non-empty. The function must return exactly one integer customer ID selected from feasible\_customers, not the depot, an index position, score, array, list, tuple, None, NaN, or infinity.
}

\vspace{10pt}
\hrule
\vspace{10pt}

\textbf{\large Penalty Construction
(\texttt{update\_edge\_distance})}

\vspace{6pt}
\textbf{\color{brown}self.prompt\_task2}\par\noindent
{\ttfamily
Given the current edge distance matrix, locally optimal CVRP routes, edge usage counts, customer demands, and vehicle capacity, design a signed global edge-guidance strategy for CVRP. The returned matrix should express how bad or good it is to include each edge in a feasible solution: Here gap means updated\_edge\_distance - edge\_distance: positive gap means the edge should be selectively penalized or avoided, with rank-discriminative signals sufficiently large to affect top-k perturbation while remaining finite and bounded relative to edge distances, negative gap means the edge should be encouraged or preserved, and zero gap means no guidance.
}

\vspace{6pt}
\textbf{\color{brown}self.prompt\_func\_name2}\par\noindent
\texttt{update\_edge\_distance}

\vspace{6pt}
\textbf{\color{brown}self.prompt\_func\_inputs2}\par\noindent
\texttt{edge\_distance},
\texttt{local\_opt\_routes},
\texttt{edge\_n\_used},
\texttt{demands},
\texttt{vehicle\_capacity}

\vspace{6pt}
\textbf{\color{brown}self.prompt\_func\_outputs2}\par\noindent
\texttt{updated\_edge\_distance}

\vspace{6pt}
\textbf{\color{brown}self.prompt\_inout\_inf2}\par\noindent
{\ttfamily
The function signature must exactly be def update\_edge\_distance. 'edge\_distance' and 'edge\_n\_used' are two-dimensional arrays with shape (n, n). 'local\_opt\_routes' is a two-dimensional integer array with shape (n\_routes, max\_route\_len); each row lists customer IDs in visiting order, uses -1 as padding, and does not include the depot. Also score depot boundary edges as well as customer-to-customer edges. Before processing each route, use customers = route[route >= 0] and only process customers; never use np.count\_nonzero(route) to determine its valid length. 'demands' is a one-dimensional array and demands[0] = 0. 'vehicle\_capacity' is a positive integer. The function must return updated\_edge\_distance as a numeric NumPy-compatible matrix with exactly the same shape as edge\_distance, not a scalar, list, tuple, None, NaN, or infinity. Do not modify edge\_n\_used in-place. The returned matrix should be symmetric, finite, non-negative, and keep the diagonal equal to zero. Avoid numerically explosive formulas such as 1/(edge\_n\_used + 1e-6), edge\_n\_used**2. The returned matrix is also used as the guided distance matrix for endpoint-driven local search in the perturbation phase.
}

\vspace{10pt}
\hrule
\vspace{10pt}

\textbf{\color{brown}self.prompt\_other\_inf}\par\noindent
{\ttfamily
Node 0 is the depot. Import the necessary Python components. When calling any function, match the number, names and types of arguments to the function signature. Don't miss the 'return' at the end. Avoid None, NaN or infinite. In order to avoid divide-by-zero, add a small epsilon to denominators every dividing. Use only variables from the current function signature or values locally computed from them; initialize every auxiliary variable on all execution paths before use. Array inputs are NumPy arrays: use NumPy-compatible operations or explicit conversion, respect the documented dimensions, distinguish node IDs from array positions, and check length/shape before indexing or broadcasting. Never use edge\_n\_used in select\_next\_node; edge\_n\_used may only be used in update\_edge\_distance. Avoid overly aggressive heuristics: do not overemphasize tiny demands or use high-order powers that make scores or penalties extreme and unstable.
}
\end{tcolorbox}

\subsection{LLM-HCJG Initialization (i1) Prompt Specification}
\label{apx:i1}

\begin{tcolorbox}[
    title={Initialization (i1)},
    width=\textwidth,
    breakable
]

\textbf{Evolution-specific hints}\par\noindent
{\ttfamily
You are an expert in heuristic algorithm design for the following
optimization problem
\textbf{\color{brown}self.problem\_description}.
}

\vspace{10pt}
\hrule
\vspace{10pt}

\textbf{Task description}\par\noindent
{\ttfamily
The first task called 'Task1' is that
\textbf{\color{brown}self.prompt\_task1}. Firstly, start with the
title 'Algorithm1:' to describe your new algorithm and main steps
in one sentence. The description must be inside a brace. Next, start
with the title 'Code1:' to implement it in Python as a function named
\textbf{\color{brown}self.prompt\_func\_name1}. This function should
accept
\textbf{str(len(\textcolor{brown}{self.prompt\_func\_inputs1}))}
input(s):
\textbf{\color{brown}self.prompt\_func\_inputs1}. The function should
return
\textbf{str(len(\textcolor{brown}{self.prompt\_func\_outputs1}))}
output(s):
\textbf{\color{brown}self.prompt\_func\_outputs1}.
\textbf{\color{brown}self.prompt\_inout\_inf1} +
\textbf{\color{brown}self.prompt\_other\_inf1}.

The second task called 'Task2' is that
\textbf{\color{brown}self.prompt\_task2}. Firstly, start with the
title 'Algorithm2:' to describe your new algorithm and main steps
in one sentence. The description must be inside a brace. Next, start
with the title 'Code2:' to implement it in Python as a function named
\textbf{\color{brown}self.prompt\_func\_name2}. This function should
accept
\textbf{str(len(\textcolor{brown}{self.prompt\_func\_inputs2}))}
input(s):
\textbf{\color{brown}self.prompt\_func\_inputs2}. The function should
return
\textbf{str(len(\textcolor{brown}{self.prompt\_func\_outputs2}))}
output(s):
\textbf{\color{brown}self.prompt\_func\_outputs2}.
\textbf{\color{brown}self.prompt\_inout\_inf2} +
\textbf{\color{brown}self.prompt\_other\_inf2}.
}

\vspace{10pt}
\hrule
\vspace{10pt}

\textbf{Associations between components identification}\par\noindent
{\ttfamily
Output `Shared Blueprint` consists of a single sentence articulating a common search-direction design principle between Task 1 and Task 2. It should define their complementary roles and explain how both component-level algorithms jointly instantiate this principle, rather than operate independently.
}

\vspace{10pt}
\hrule
\vspace{10pt}

\textbf{Other hints}\par\noindent
{\ttfamily
Do not give additional explanations.
}

\end{tcolorbox}

\subsection{LLM-HCJG Crossover (e1) Prompt Specification}
\label{apx:e1}

\begin{tcolorbox}[
    title={Crossover (e1)},
    width=\textwidth,
    breakable
]

\textbf{Task description}\par\noindent
{\ttfamily
There are two task-algorithms for you to design. Task 1:
\textbf{\color{brown}self.prompt\_task1}; Task 2:
\textbf{\color{brown}self.prompt\_task2}.
}

\vspace{10pt}
\hrule
\vspace{10pt}

\textbf{Parent algorithm(s)}\par\noindent
{\ttfamily
Now I have \textbf{str(len(indivs))} existing individuals with their
hared blueprints, algorithms, and code for the two tasks as follows:

{
\color{blue}\textbf{No.1 individual}\\
\color{blue}
Shared Blueprint: indivs[1][blueprint]\\
\color{blue}
For task1: Algorithm1: \textbf{indivs[1][algorithm1]} +
Code1:\textbf{indivs[1][code1]}\\
\color{blue}
For task2: Algorithm2: \textbf{indivs[1][algorithm2]} +
Code2:\textbf{indivs[1][code2]}
}

\medskip
{
\color{purple}\textbf{No.2 individual}\\
\color{purple}
Shared Blueprint: indivs[2][blueprint]\\
\color{purple}
For task1: Algorithm1: \textbf{indivs[2][algorithm1]} +
Code1:\textbf{indivs[2][code1]}\\
\color{purple}
For task2: Algorithm2: \textbf{indivs[2][algorithm2]} +
Code2:\textbf{indivs[2][code2]}
}

\medskip

...
}

\vspace{10pt}
\hrule
\vspace{10pt}

\textbf{Evolution-specific hints}\par\noindent
{\ttfamily
Please help me create a new individual with new algorithms
(Algorithm1 and Algorithm2) that have totally different forms from
the given ones.
}

\vspace{10pt}
\hrule
\vspace{10pt}

\textbf{Associations between components identification}\par\noindent
{\ttfamily
Output `Shared Blueprint` consists of a single sentence articulating a common search-direction design principle between Task 1 and Task 2. It should define their complementary roles and explain how both component-level algorithms jointly instantiate this principle, rather than operate independently.
}

\vspace{10pt}
\hrule
\vspace{10pt}

\textbf{Other hints}\par\noindent
{\ttfamily
Do not give additional explanations.
\par}

\end{tcolorbox}

\subsection{LLM-HCJG Crossover (e2) Prompt Specification}
\label{apx:e2}

\begin{tcolorbox}[
    title={Crossover (e2)},
    width=\textwidth,
    breakable
]

\textbf{Task description}\par\noindent
{\ttfamily
There are two task-algorithms for you to design. Task 1:
\textbf{\color{brown}self.prompt\_task1}; Task 2:
\textbf{\color{brown}self.prompt\_task2}.
}

\vspace{10pt}
\hrule
\vspace{10pt}

\textbf{Parent algorithm(s)}\par\noindent
{\ttfamily
Now I have \textbf{str(len(indivs))} existing individuals with their
shared blueprints, algorithms, and codes for the two tasks as follows:

{
\color{blue}\textbf{No.1 individual}\\
\color{blue}
Shared Blueprint: indivs[1][blueprint]\\
\color{blue}
For task1: Algorithm1: \textbf{indivs[1][algorithm1]} +
Code1:\textbf{indivs[1][code1]}\\
\color{blue}
For task2: Algorithm2: \textbf{indivs[1][algorithm2]} +
Code2:\textbf{indivs[1][code2]}
}

\medskip
{
\color{purple}\textbf{No.2 individual}\\
\color{purple}
Shared Blueprint: indivs[2][blueprint]\\
\color{purple}
For task1: Algorithm1: \textbf{indivs[2][algorithm1]} +
Code1:\textbf{indivs[2][code1]}\\
\color{purple}
For task2: Algorithm2: \textbf{indivs[2][algorithm2]} +
Code2:\textbf{indivs[2][code2]}
}

\medskip

...
}

\vspace{10pt}
\hrule
\vspace{10pt}

\textbf{Evolution-specific hints}\par\noindent
{\ttfamily
Please help me create a new individual with new algorithms
(Algorithm1 and Algorithm2) that have totally different forms from
the given ones, but can be inspired by them.
}

\vspace{10pt}
\hrule
\vspace{10pt}

\textbf{Associations between components identification}\par\noindent
{\ttfamily
Output `Shared Blueprint` consists of a single sentence articulating a common search-direction design principle between Task 1 and Task 2. It should define their complementary roles and explain how both component-level algorithms jointly instantiate this principle, rather than operate independently.
}

\vspace{10pt}
\hrule
\vspace{10pt}

\textbf{Other hints}\par\noindent
{\ttfamily
Do not give additional explanations.
\par}

\end{tcolorbox}

\subsection{LLM-HCJG Mutation (m1) Prompt Specification}
\label{apx:m1}

\begin{tcolorbox}[
    title={Mutation (m1)},
    width=\textwidth,
    breakable
]

\textbf{Task description}\par\noindent
{\ttfamily
There are two task-algorithms for you to design. Task 1:
\textbf{\color{brown}self.prompt\_task1}; Task 2:
\textbf{\color{brown}self.prompt\_task2}.
}

\vspace{10pt}
\hrule
\vspace{10pt}

\textbf{Parent algorithm(s)}\par\noindent
{\ttfamily
Now I have
{\color{green!60!black}\textbf{one elite individual}}
with its shared blueprint, algorithm, and code as follows:

{\color{green!60!black}
Shared Blueprint: \textbf{indivs[blueprint]}\\
For task1: Algorithm1: \textbf{indivs[algorithm1]} +
Code1:\textbf{indivs[code1]}\\
\color{green!60!black}
For task2: Algorithm2: \textbf{indivs[algorithm2]} +
Code2:\textbf{indivs[code2]}
}
}

\vspace{10pt}
\hrule
\vspace{10pt}

\textbf{Evolution-specific hints}\par\noindent
{\ttfamily
Please assist me in creating a new individual with new algorithms
(Algorithm1 and Algorithm2) that have different forms but can be
modified versions of the algorithms provided.
}

\vspace{10pt}
\hrule
\vspace{10pt}

\textbf{Associations between components identification}\par\noindent
{\ttfamily
Output `Shared Blueprint` consists of a single sentence articulating a common search-direction design principle between Task 1 and Task 2. It should define their complementary roles and explain how both component-level algorithms jointly instantiate this principle, rather than operate independently.
}

\vspace{10pt}
\hrule
\vspace{10pt}

\textbf{Other hints}\par\noindent
{\ttfamily
Do not give additional explanations.
\par}

\end{tcolorbox}

\subsection{LLM-HCJG Mutation (m2) Prompt Specification}
\label{apx:m2}

\begin{tcolorbox}[
    title={Mutation (m2)},
    width=\textwidth,
    breakable
]

\textbf{Task description}\par\noindent
{\ttfamily
There are two task-algorithms for you to design. Task 1:
\textbf{\color{brown}self.prompt\_task1}; Task 2:
\textbf{\color{brown}self.prompt\_task2}.
}

\vspace{10pt}
\hrule
\vspace{10pt}

\textbf{Parent algorithm(s)}\par\noindent
{\ttfamily
Now I have
{\color{green!60!black}\textbf{one elite individual}}
with its shared blueprint, algorithm, and code as follows:

{
\textcolor{green!60!black}{
  Shared Blueprint: \textbf{indivs[blueprint]}
}\\
\textcolor{green!60!black}{
  For task1: Algorithm1: \textbf{indivs[algorithm1]} +
  Code1: \textbf{indivs[code1]}
}\\
\textcolor{green!60!black}{
  For task2: Algorithm2: \textbf{indivs[algorithm2]} +
  Code2: \textbf{indivs[code2]}
}
}
}

\vspace{10pt}
\hrule
\vspace{10pt}

\textbf{Evolution-specific hints}\par\noindent
{\ttfamily
Please identify the main algorithm parameters and assist me in
creating a new individual with new algorithms (Algorithm1 and
Algorithm2) that have different parameter settings of the score
functions provided.
}

\vspace{10pt}
\hrule
\vspace{10pt}

\textbf{Associations between components identification}\par\noindent
{\ttfamily
Output `Shared Blueprint` consists of a single sentence articulating a common search-direction design principle between Task 1 and Task 2. It should define their complementary roles and explain how both component-level algorithms jointly instantiate this principle, rather than operate independently.
}

\vspace{10pt}
\hrule
\vspace{10pt}

\textbf{Other hints}\par\noindent
\begingroup
\ttfamily
Do not give additional explanations.
\par
\endgroup

\end{tcolorbox}

\subsection{LLM-HCJG Mutation (m3) Prompt Specification}
\label{apx:m3}

\begin{tcolorbox}[
    title={Mutation (m3)},
    width=\textwidth,
    breakable
]

\textbf{Task description}\par\noindent
{\ttfamily
There are two task-algorithms for you to design. Task 1: \textbf{\color{brown}self.prompt\_task1}; Task 2:
\textbf{\color{brown}self.prompt\_task2}.
}

\vspace{10pt}
\hrule
\vspace{10pt}

\textbf{Parent algorithm(s)}\par\noindent
{\ttfamily
Now I have
{\color{green!60!black}\textbf{one elite individual}}
with its shared blueprint, algorithm, and code as follows:

{\color{green!60!black}
Shared Blueprint: \textbf{indivs[blueprint]}\\
\color{green!60!black}
For task1: Algorithm1: \textbf{indivs[algorithm1]} +
Code1:\textbf{indivs[code1]}\\
\color{green!60!black}
For task2: Algorithm2: \textbf{indivs[algorithm2]} +
Code2:\textbf{indivs[code2]}
}
}

\vspace{10pt}
\hrule
\vspace{10pt}

\textbf{Evolution-specific hints}\par\noindent
{\ttfamily
Please assist me in creating a new individual with new algorithms
(Algorithm1 and Algorithm2) that are simplified versions of the
provided algorithms, focusing on reducing unnecessary complexity
while preserving their essential functionality.
}

\vspace{10pt}
\hrule
\vspace{10pt}

\textbf{Associations between components identification}\par\noindent
{\ttfamily
Output `Shared Blueprint` consists of a single sentence articulating a common search-direction design principle between Task 1 and Task 2. It should define their complementary roles and explain how both component-level algorithms jointly instantiate this principle, rather than operate independently.
}

\vspace{10pt}
\hrule
\vspace{10pt}

\textbf{Other hints}\par\noindent
{\ttfamily
Do not give additional explanations.
\par}

\end{tcolorbox}

\section{Guided Local Search Realizations for Routing Optimization}
\label{app:routing-gls-frameworks}

\subsection{Canonical GLS Procedure for TSP}
\label{apx:tsp-gls}

\begin{algorithm}[!ht]
\normalsize
\caption{Guided Local Search for TSP}
\label{alg:gls}
\begin{algorithmic}[1]
  \Require Distance matrix $ \mathcal{D}$, maximum iterations $T$, scaling coefficient $\alpha$, current route $S$, edge penalty $ \mathcal{P}$ and number of customers $n$
  \Ensure Best route $S^*$       

  \State $S \gets \textbf{Nearest\_neighbor\_2End}( \mathcal{D})$
  \State $S \gets  \textbf{Local\_search}_{\text{2-opt+relocate}}(S, \mathcal{D})$
  \State $S^* \gets S$, $cost^* \gets \textbf{Cost}( \mathcal{D},S)$
  \State $N \gets |S|$ 
  \State $\mathcal{P} \gets \mathbf{0}$ 
  \State $\lambda \gets \alpha \cdot \dfrac{cost^*}{n}$

  \For{$k = 1$ \textbf{to} $T$}
    \State $u_{\max} \gets -\infty$
    \State $(i^\star, j^\star) \gets \emptyset$
    \ForAll{$(i, j)$ in $S$}
      \State $u \gets \dfrac{d_{ij}}{1+p_{ij}}$
      \If{$u > u_{\max}$}
        \State $u_{\max} \gets u$
        \State $(i^\star, j^\star) \gets (i, j)$
      \EndIf
    \EndFor

     \State $\mathcal{P}[i^\star,j^\star]\!+\!=1$, $\mathcal{P}[j^\star,i^\star]\!+\!=1$ 

    \State $\mathcal{D}^{\prime} \gets \mathcal{D} + \lambda \cdot \mathcal{P}$
    \State $S \gets  \textbf{Local\_search}_{\text{2-opt+relocate}}(S, \mathcal{D}^{\prime})$

    \State $cost \gets \textbf{Cost}(\mathcal{D}, S)$ 
    \If{$cost < cost^*$}
      \State $S^* \gets S$, $cost^* \gets $ $cost$
    \EndIf
  \EndFor

  \State \Return $S^*$
\end{algorithmic}
\end{algorithm}

\clearpage
\subsection{LLM-Driven GLS Procedure for TSP}
\label{apx:llmgls-tsp}

\begin{algorithm}[!ht]
\normalsize
\caption{LLM-driven Guided Local Search for TSP}
\label{alg:LLMgls}
\begin{algorithmic}[1]
  \Require Distance matrix $\mathcal{D}$, maximum iterations $T$, maximum perturbation moves per iteration $p$, current route $S$ and edge penalty $\mathcal{P}$
  \Ensure Best route $S^*$   

\State $S \gets \textbf{LLM\_select\_next\_node}(\mathcal{D}) $  \hfill \texttt{// LLM Solution initialization}
  \State $S \gets \textbf{Local\_search}_{\text{2-opt+relocate}}(S,\mathcal{D})$
  \State $S^* \gets S$, $cost^* \gets \textbf{Cost}(\mathcal{D},S)$
  
  \State $\mathcal{P} \gets \mathbf{0}$ 
  \For{$k = 1$ \textbf{to} $T$}
    \For{$h = 1$ \textbf{to} $p$}
      \State $D^{\prime} \gets \textbf{LLM\_update\_edge\_distance}(\mathcal{D},\,S,\,\mathcal{P})$  \hfill \texttt{// LLM Penalty construction}

      \State $\Delta \gets max(\mathcal{D}^{\prime}-\mathcal{D},0)$
      \State $\mathcal{E} \gets \text{Top5Edges}(\Delta)$

      \ForAll{edge $(i,j)$ in $\mathcal{E}$}
        \State $\mathcal{P}[i,j]\!+\!=1$, $\mathcal{P}[j,i]\!+\!=1$ 
        \State $\delta,S' \gets \textbf{Two-Opt}(S,i,j,\mathcal{D}^{\prime})$
        \If{$\delta<0$}
          \State $S \gets S'$, $cost \gets \textbf{Cost}(\mathcal{D},S)$
        \EndIf

        \State $\delta,S' \gets \textbf{Relocate}(S,i,j,\mathcal{D}^{\prime})$
        \If{$\delta<0$}
          \State $S \gets S'$, $cost \gets \textbf{Cost}(\mathcal{D},S)$
        \EndIf
      \EndFor
    \EndFor
    \State $S \gets \textbf{Local\_search}_{\text{2-opt+relocate}}(S,\mathcal{D})$
    \State $cost \gets \textbf{Cost}(\mathcal{D},S)$
    \If{$cost < cost^*$}
      \State $S^* \gets S$, $cost^* \gets cost$
    \EndIf
    
    \If{$k \bmod 50 = 0$}
      \State $S,\,cost \gets S^*,\,cost^*$
    \EndIf
   
  \EndFor
  \State \Return $S^*$
\end{algorithmic}
\end{algorithm}

\clearpage
\subsection{Transferred GLS Procedure for CVRP}
\label{apx:cvrp-gls}

\begin{algorithm}[!ht]
\caption{Guided Local Search for CVRP}
\label{alg:cvrp_gls}
\begin{algorithmic}[1]
  \Require Distance matrix $\mathcal{D}$, demand vector $q$, vehicle capacity $Q$, required number of vehicles $K$, maximum iterations $T$, scaling coefficient $\alpha$, route set $R$, edge penalty $\mathcal{P}$ and number of customers $n$
  \Ensure Best route set $R^*$

  \State $R \gets \textbf{Nearest\_neighbor\_2End}(\mathcal{D},q,Q,K)$
  \State $R \gets \textbf{Local\_search}_{\text{2-opt+relocate}}(R,\mathcal{D},q,Q)$
  \State $R^* \gets R$, $cost^* \gets \textbf{Cost}(\mathcal{D},R)$
  \State $\mathcal{P} \gets \mathbf{0}$
  \State $\lambda \gets \alpha \cdot \dfrac{cost^*}{n}$

  \For{$k = 1$ \textbf{to} $T$}
      \State $u_{\max} \gets -\infty$
      \State $(i^\star,j^\star) \gets \emptyset$
      \ForAll{$(i,j)$ in $R$}
        \State $u \gets \dfrac{d_{ij}}{1+\mathcal{P}_{ij}}$
        \If{$u > u_{\max}$}
          \State $u_{\max} \gets u$
          \State $(i^\star,j^\star) \gets (i,j)$
        \EndIf
      \EndFor
      
      \State $\mathcal{P}[i^\star,j^\star]\!+\!=1$, $\mathcal{P}[j^\star,i^\star]\!+\!=1$
      \State $\mathcal{D}^{\prime} \gets \mathcal{D}+\lambda\cdot\mathcal{P}$
      \State $R \gets \textbf{Local\_search}_{\text{2-opt+relocate+swap}}(R,\mathcal{D}^{\prime},q,Q)$

    \State $cost \gets \textbf{Cost}(\mathcal{D},R)$
    \If{$cost < cost^*$}
      \State $R^* \gets R$, $cost^* \gets cost$
    \EndIf
  \EndFor

  \State \Return $R^*$
\end{algorithmic}
\end{algorithm}

\subsection{LLM-Driven GLS Procedure for CVRP}
\label{apx:llmgls-cvrp}

\begin{algorithm}[H]
\normalsize
\caption{LLM-driven Guided Local Search for CVRP}
\label{alg:LLMgls_cvrp}
\begin{algorithmic}[1]
  \Require Distance matrix $\mathcal{D}$, customer demands $q$, vehicle capacity $Q$, required vehicle number $K$, maximum iterations $T$, maximum perturbation moves per iteration $p$, current route set $R$ and edge penalty $\mathcal{P}$
  \Ensure Best route set $R^*$

  \State $R \gets \textbf{LLM\_select\_next\_node}(\mathcal{D},q,Q,K)$ 
  \hfill \texttt{// LLM feasible solution initialization}
  \State $R \gets \textbf{Local\_search}_{\text{2-opt+relocate}}(R,\mathcal{D},q,Q)$
  \State $R^* \gets R$, $cost^* \gets \textbf{Cost}(\mathcal{D},R)$

  \State $\mathcal{P} \gets \mathbf{0}$
  \For{$k = 1$ \textbf{to} $T$}
    \For{$h = 1$ \textbf{to} $p$}
      \State $\mathcal{D}^{\prime} \gets \textbf{LLM\_update\_edge\_distance}(\mathcal{D},R,\mathcal{P},q,Q)$
      \hfill \texttt{// LLM Penalty construction}

      \State $\Delta \gets \max(\mathcal{D}^{\prime}-\mathcal{D},0)$
      \State $\mathcal{E} \gets \text{Top5Edges}(\Delta)$

      \ForAll{edge $(i,j)$ in $\mathcal{E}$}
        \State $\mathcal{P}[i,j]\!+\!=1$, $\mathcal{P}[j,i]\!+\!=1$
        \State $\delta,R' \gets 
        \textbf{Two-Opt}(R,i,j,\mathcal{D}^{\prime})$
        \If{$\delta<0$}
            \State $R \gets R'$, $cost \gets \textbf{Cost}(\mathcal{D},R)$
          \EndIf

          \State $\delta,R' \gets \textbf{Relocate}(R,i,j,\mathcal{D}^{\prime},q,Q)$
          \If{$\delta<0$}
            \State $R \gets R'$, $cost \gets \textbf{Cost}(\mathcal{D},R)$
          \EndIf

          \State $\delta,R' \gets
          \textbf{Swap}(R,i,j,\mathcal{D}^{\prime},q,Q)$
          \If{$\delta<0$}
            \State $R \gets R'$, $cost \gets \textbf{Cost}(\mathcal{D},R)$
          \EndIf
        \EndFor
      \EndFor

    \State $R \gets \textbf{Local\_search}_{\text{2-opt+relocate}}(R,\mathcal{D},q,Q)$
    \State $cost \gets \textbf{Cost}(\mathcal{D},R)$
    \If{$cost < cost^*$}
      \State $R^* \gets R$, $cost^* \gets cost$
    \EndIf

    \If{$k \bmod 50 = 0$}
      \State $R,\,cost \gets R^*,\,cost^*$
    \EndIf

  \EndFor

  \State \Return $R^*$
\end{algorithmic}
\end{algorithm}

\section{Best Elite Heuristics Automatically Generated by LLM}
\label{app:best-elite-heuristics}
\subsection{Best LLM-Generated Elite Heuristic for TSP}
\label{apx:tsp-elite}

\begin{tcolorbox}[
    title={The Best Elite Individual},
    width=\textwidth,
    breakable
]
\textbf{Shared Blueprint}\par\noindent
\texttt{Distance-driven stochastic guidance: construct the initial route by balancing immediate travel cost with randomized continuation potential, then diversify subsequent search through distance-scaled, usage-adaptive randomized penalties on current-route edges.}

\vspace{10pt}
\hrule
\vspace{10pt}

\textbf{Thought1}\par\noindent
\texttt{Implement a distance-driven stochastic greedy rule to select the next node: evaluate each candidate mainly by its immediate distance from the current node, while adding a random-weighted nearest-continuation distance to perturb the greedy choice and encourage exploration.}

\vspace{10pt}
\hrule
\vspace{10pt}

\textbf{Code1}\par\noindent
\begin{Verbatim}[
    breaklines,
    breakanywhere=true,
    commandchars=\\\{\}
]
\textcolor{red}{import} numpy \textcolor{red}{as} np

\textcolor{red}{def select_next_node}(current_node, destination_node, unvisited_nodes, distance_matrix):
    scores = []
    for next_node in unvisited_nodes:
        distance_to_next = distance_matrix[current_node][next_node]
        min_distance_to_unvisited = np.min([
            distance_matrix[next_node][unvisited]
            for unvisited in unvisited_nodes
            if unvisited != next_node
        ])
        score = distance_to_next + min_distance_to_unvisited * np.random.rand()
        scores.append(score)

    next_node = unvisited_nodes[np.argmin(scores)]
    \textcolor{red}{return} next_node
\end{Verbatim}

\vspace{10pt}
\hrule
\vspace{10pt}

\textbf{Thought2}\par\noindent
\texttt{Implement the distance-driven stochastic perturbation principle to update the search landscape: penalize edges in the local-optimal tour proportionally to their current distance, amplify the penalty for less-used edges, and add random noise to diversify subsequent search moves.}

\vspace{10pt}
\hrule
\vspace{10pt}

\textbf{Code2}\par\noindent
\begin{Verbatim}[
    breaklines,
    breakanywhere=true,
    commandchars=\\\{\}
]
\textcolor{red}{import} numpy \textcolor{red}{as} np

\textcolor{red}{def update_edge_distance}(edge_distance, local_opt_tour, edge_n_used):
    updated_edge_distance = np.copy(edge_distance)

    for i in range(len(local_opt_tour)):
        start_node = local_opt_tour[i]
        end_node = local_opt_tour[(i + 1) % len(local_opt_tour)]

        usage_count = edge_n_used[start_node][end_node] + 1e-8
        distance = updated_edge_distance[start_node][end_node]
        penalty = distance * (1 + (1 / usage_count)) * (1 + np.random.rand())

        updated_edge_distance[start_node][end_node] += penalty
        updated_edge_distance[end_node][start_node] += penalty

    \textcolor{red}{return} updated_edge_distance
\end{Verbatim}

\vspace{10pt}
\hrule
\vspace{10pt}

\textbf{Fitness}\par\noindent
\texttt{0.000}

\end{tcolorbox}

\medskip
\noindent\textbf{Mathematical interpretation of the generated functions.}\par\noindent
In \texttt{LLM\_select\_next\_node}, let \(i\) be the current node and \(U\) the set of unvisited nodes. Each candidate \(j \in U\) is scored by combining its immediate distance with a randomized nearest-continuation term,
\begin{equation}
s(j \mid i,U)
=
d_{ij}
+
\xi_j\cdot\min_{k \in U \setminus \{j\}} d_{jk},
\qquad
\xi_j \sim \mathcal{U}(0,1).
\end{equation}
The next node is selected as \(j^*=\arg\min_{j\in U}s(j\mid i,U)\).

In \texttt{LLM\_update\_edge\_distance}, for each edge \((i,j)\) in the current locally optimal tour, a distance-scaled, usage-adaptive, and randomized penalty is computed as
\begin{equation}
\Delta d_{ij}
=
d_{ij} \cdot
\left(1+\frac{1}{p_{ij}+\varepsilon}\right) \cdot
(1+\eta_{ij}),
\qquad
\eta_{ij} \sim \mathcal{U}(0,1),
\end{equation}
where \(d_{ij}\) denotes the current edge distance in the updated distance matrix, \(p_{ij}\) is the historical usage count of edge \((i,j)\), and \(\varepsilon\) avoids division by zero. The penalized distance matrix \(\mathcal{D}^{\prime}\) is then symmetrically modified by
\begin{equation}
d'_{ij}
=
\max\left\{
0,\,
d_{ij}+\Delta d_{ij}
\right\},
\qquad
d'_{ji}=d'_{ij}.
\tag{30}
\end{equation}

\subsection{Best LLM-Generated Elite Heuristic for CVRP}
\label{apx:cvrp-elite}

\begin{tcolorbox}[
    title={The Best Elite Individual},
    width=\textwidth,
    breakable
]

\textbf{Shared Blueprint}\par\noindent
\texttt{Structure-aware penalty-augmented distance shaping: construct feasible routes by balancing immediate distance with spatial coherence and long-move avoidance, then guide perturbation through demand- and usage-aware adjustments to current-route edges.}

\vspace{10pt}
\hrule
\vspace{10pt}

\textbf{Thought1}\par\noindent
\texttt{Select the next feasible customer using a penalty-augmented distance score that combines immediate travel distance, spatial coherence among remaining feasible customers, and an additional penalty for unusually long moves.}

\vspace{10pt}
\hrule
\vspace{10pt}
\textbf{Code1}\par\noindent
\begin{Verbatim}[
    breaklines,
    breakanywhere=true,
    commandchars=\\\{\}
]
\textcolor{red}{import} numpy \textcolor{red}{as} np

\textcolor{red}{def select_next_node}(current_node, feasible_customers, remaining_customers, remaining_capacity, demands, distance_matrix):
    next_node = None
    best_score = float('inf')

    distances = distance_matrix[current_node, feasible_customers]

    for i, customer in enumerate(feasible_customers):
        if demands[customer] <= remaining_capacity:
            nearby_customers = [
                c for c in remaining_customers
                if c != customer and demands[c] <= remaining_capacity
            ]

            if len(nearby_customers) > 0:
                spatial_coherence = np.mean(
                    distance_matrix[customer, nearby_customers]
                )
            else:
                spatial_coherence = 0.0

            score = distances[i] + 0.15 * spatial_coherence

            if distances[i] > np.mean(distances):
                score *= 1.1

            if score < best_score:
                best_score = score
                next_node = customer

    \textcolor{red}{return} next_node
\end{Verbatim}

\vspace{10pt}
\hrule
\vspace{10pt}
\textbf{Thought2}\par\noindent
\texttt{Update the edge distance matrix using a penalty-augmented distance adjustment, where each local-optimal-route edge is scaled by its distance, route demand level, and historical usage, with clipping and symmetry preserving stability.}

\vspace{10pt}
\hrule
\vspace{10pt}
\textbf{Code2}\par\noindent
\begin{Verbatim}[
    breaklines,
    breakanywhere=true,
    commandchars=\\\{\}
]
\textcolor{red}{import} numpy \textcolor{red}{as} np

\textcolor{red}{def update_edge_distance}(edge_distance, local_opt_routes, edge_n_used, demands, vehicle_capacity):
    n = edge_distance.shape[0]
    updated_edge_distance = np.copy(edge_distance)

    for route in local_opt_routes:
        customers = route[route >= 0]

        if customers.size > 0:
            total_demand = demands[customers].sum()
            demand_normalized = total_demand / max(vehicle_capacity, 1)

            for i in range(len(customers) - 1):
                from_customer = customers[i]
                to_customer = customers[i + 1]

                edge_use = (
                    edge_n_used[from_customer][to_customer] + 1e-6
                )

                adjustment = (
                    edge_distance[from_customer][to_customer]
                    * (1 - edge_use / 20)
                    * (1 + 0.5 * demand_normalized)
                )

                updated_edge_distance[from_customer][to_customer] = max(
                    0,
                    updated_edge_distance[from_customer][to_customer]
                    + adjustment
                )

                updated_edge_distance[to_customer][from_customer] = (
                    updated_edge_distance[from_customer][to_customer]
                )

    np.fill_diagonal(updated_edge_distance, 0)

    \textcolor{red}{return} updated_edge_distance
\end{Verbatim}

\vspace{10pt}
\hrule
\vspace{10pt}
\textbf{Fitness}\par\noindent
\texttt{2.046}

\end{tcolorbox}

\medskip
\noindent\textbf{Mathematical interpretation of the generated functions.}\par\noindent
In \texttt{LLM\_select\_next\_node}, let \(i\) be the current node and \(\mathcal{C}\) the set of capacity-feasible candidate customers. For each \(j\in\mathcal{C}\), the compactness-oriented score is defined as
\begin{equation}
s_j=
\gamma_j
\left(
d_{ij}
+
0.15\,c_j
\right),
\qquad
c_j=
\begin{cases}
\frac{1}{|\mathcal{N}_j|}\sum_{k\in\mathcal{N}_j}d_{jk}, & |\mathcal{N}_j|>0,\\
0, & |\mathcal{N}_j|=0,
\end{cases}
\end{equation}
where \(d_{ij}\) is the immediate travel distance, \(c_j\) is the spatial-coherence penalty computed over \(\mathcal{N}_j=\mathcal{C}\setminus\{j\}\), and \(\gamma_j\) is the long-move penalty factor:
\begin{equation}
\gamma_j=
\begin{cases}
1.1, & d_{ij}>\frac{1}{|\mathcal{C}|}\sum_{\ell\in\mathcal{C}}d_{i\ell},\\
1, & \text{otherwise}.
\end{cases}
\end{equation}
The next customer is selected as \(j^*=\arg\min_{j\in\mathcal{C}}s_j\). 

In \texttt{LLM\_update\_edge\_distance}, for each local-optimal route \(R\), let \(C(R)\) denote its valid customer sequence and define the normalized route demand as
\begin{equation}
\rho_R=\frac{\sum_{v\in C(R)}q_v}{\max(Q,1)}.
\end{equation}
For each edge \((i,j)\) in \(C(R)\), a structure-aware distance adjustment reshapes the local-route edge cost by scaling the original distance with route demand and historical edge usage:
\begin{equation}
\Delta d_{ij}^{R}
=
d_{ij}
\left(1-\frac{p_{ij}+\varepsilon}{20}\right)
\left(1+0.5\rho_R\right),
\end{equation}
where \(p_{ij}\) is the historical usage count of edge \((i,j)\), and \(\varepsilon\) is a small smoothing constant. The penalized distance matrix \(\mathcal{D}^{\prime}\) is then symmetrically modified by
\begin{equation}
d_{ij}^{\prime R}
=
\max\left\{
0,\,
d_{ij}^{R}+\Delta d_{ij}^{R}
\right\},
\qquad
d_{ji}^{\prime R}=d_{ij}^{\prime R}.
\tag{35}
\end{equation}

\subsection{Best Penalty-Construction Individuals Generated by EoH and ReEvo}
\label{apx:eoh-reevo-best-penalty}

\subsubsection{EoH on TSP from Publication \citep{9e7eceac6e31432aafced5e02a49de16}}
\label{apx:eoh-tsp-best-penalty}

\begin{tcolorbox}[
    title={The Best EoH Penalty-Construction Individual for TSP},
    width=\textwidth,
    breakable,
    fontupper=\small,
    before skip=3pt,
    after skip=3pt
]

\textbf{Thought2}\par\noindent
\texttt{Update the edge distances in the edge distance matrix by incorporating a pheromone-like effect, where the update is determined by edge count, distance, and usage, with the addition of a decay factor to avoid stagnation and promote exploration.}

\vspace{5pt}
\hrule
\vspace{5pt}

\textbf{Code2}\par\noindent
\begin{Verbatim}[
    breaklines,
    breakanywhere=true,
    commandchars=\\\{\}
]
\textcolor{red}{import} numpy \textcolor{red}{as} np

\textcolor{red}{def update_edge_distance}(edge_distance, local_opt_tour, edge_n_used):
    updated_edge_distance = np.copy(edge_distance)
    edge_count = np.zeros_like(edge_distance)
    for i in range(len(local_opt_tour) - 1):
        start = local_opt_tour[i]
        end = local_opt_tour[i + 1]
        edge_count[start][end] += 1
        edge_count[end][start] += 1

    edge_n_used_max = np.max(edge_n_used) + 1e-8
    decay_factor = 0.1
    mean_distance = np.mean(edge_distance) + 1e-8

    for i in range(edge_distance.shape[0]):
        for j in range(edge_distance.shape[1]):
            if edge_count[i][j] > 0:
                noise_factor = (
                    np.random.uniform(0.7, 1.3) / edge_count[i][j]
                    + edge_distance[i][j] / mean_distance
                    - (0.3 / edge_n_used_max) * edge_n_used[i][j]
                )
                updated_edge_distance[i][j] += (
                    noise_factor * (1 + edge_count[i][j])
                    - decay_factor * updated_edge_distance[i][j]
                )

    \textcolor{red}{return} updated_edge_distance
\end{Verbatim}

\end{tcolorbox}
\subsubsection{EoH on CVRP}
\label{apx:eoh-cvrp-best-penalty}

\begin{tcolorbox}[
    title={The Best EoH Penalty-Construction Individual for CVRP},
    width=\textwidth,
    breakable,
    fontupper=\small,
    before skip=3pt,
    after skip=3pt
]

\textbf{Thought2}\par\noindent
\texttt{Balancing capacity awareness and search diversification by revising traversed edge costs according to customer demand, distance, and historical usage.}

\vspace{5pt}
\hrule
\vspace{5pt}

\textbf{Code2}\par\noindent
\begin{Verbatim}[
    breaklines,
    breakanywhere=true,
    commandchars=\\\{\}
]
\textcolor{red}{def update_edge_distance}(edge_distance, local_opt_routes, edge_n_used, demands, vehicle_capacity):
    updated_edge_distance = edge_distance.copy()
    eps = 1e-8

    for route in local_opt_routes:
        if route[0] == -1:
            continue

        route_length = len(route)

        for i in range(route_length - 1):
            if route[i] == -1 or route[i + 1] == -1:
                continue

            u = route[i]
            v = route[i + 1]
            current_load = demands[u]

            distance_penalty = edge_distance[u][v] + eps
            usage_penalty = edge_n_used[u][v] + eps

            if vehicle_capacity > 0:
                updating_value = (
                    current_load / vehicle_capacity
                ) * (distance_penalty + usage_penalty)
            else:
                updating_value = distance_penalty + usage_penalty

            updated_edge_distance[u][v] += updating_value
            updated_edge_distance[v][u] += updating_value
            edge_n_used[u][v] += 1
            edge_n_used[v][u] += 1

    np.fill_diagonal(updated_edge_distance, 0)
    \textcolor{red}{return} updated_edge_distance
\end{Verbatim}

\end{tcolorbox}
\subsubsection{ReEvo on TSP}
\label{apx:reevo-tsp-best-penalty}

\begin{tcolorbox}[
    title={The Best ReEvo Penalty-Construction Individual for TSP},
    width=\textwidth,
    breakable,
    fontupper=\normalsize,
    before skip=3pt,
    after skip=3pt
]

\textbf{Thought2}\par\noindent
\texttt{Adaptively reshaping edge distances through usage-aware nonlinear penalties and exploration noise while preserving symmetry, non-negativity, and the original overall scale.}

\vspace{5pt}
\hrule
\vspace{5pt}

\textbf{Code2}\par\noindent
\begin{Verbatim}[
    breaklines,
    breakanywhere=true,
    commandchars=\\\{\}
]
\textcolor{red}{import} numpy \textcolor{red}{as} np

\textcolor{red}{def update_edge_distance}(edge_distance: np.ndarray,
                         local_opt_tour: np.ndarray,
                         edge_n_used: np.ndarray) -> np.ndarray:
    import numpy as np

    updated = edge_distance.copy()
    n = len(local_opt_tour)
    max_penalty = np.max(edge_distance) + 1

    for i in range(n):
        u = int(local_opt_tour[i])
        v = int(local_opt_tour[(i + 1) % n])

        usage_count = edge_n_used[u, v]
        distance = edge_distance[u, v]

        penalty = ( 0.1 * distance * (1 + 0.05 * usage_count ** 2)
            if usage_count > 0
            else 0.1 * distance + max_penalty
        )

        updated[u, v] = max(0, distance + penalty)
        updated[v, u] = updated[u, v]

        if usage_count > 0:
            decay_factor = 0.02 * np.log(usage_count + 1) ** 2
            updated[u, v] += decay_factor
            updated[v, u] += decay_factor

        if usage_count == 0:
            underused_penalty = 0.2 * (1 / (1 + np.sqrt(distance)))
            updated[u, v] += underused_penalty
            updated[v, u] += underused_penalty

    exploration_need = np.mean(edge_n_used)
    randomness_scale = max(0.1, 0.5 / (1 + exploration_need))
    randomness = np.random.rand(n, n) * randomness_scale
    updated += randomness

    updated = np.maximum(updated, 0)

    scale_factor = np.sum(updated) + 1e-10
    updated = updated / scale_factor * np.sum(edge_distance)

    np.fill_diagonal(updated, 0)

    \textcolor{red}{return} updated
\end{Verbatim}

\end{tcolorbox}
\subsubsection{ReEvo on CVRP}
\label{apx:reevo-cvrp-best-penalty}

\begin{tcolorbox}[
    title={The Best ReEvo Penalty-Construction Individual for CVRP},
    width=\textwidth,
    breakable,
    fontupper=\small,
    before skip=3pt,
    after skip=3pt
]

\textbf{Thought2}\par\noindent
\texttt{Diversifying routes by nonlinearly increasing the costs of unique route edges according to repeated usage, endpoint demand, and total route load.}

\vspace{5pt}
\hrule
\vspace{5pt}

\textbf{Code2}\par\noindent
\begin{Verbatim}[
    breaklines,
    breakanywhere=true,
    commandchars=\\\{\}
]
\textcolor{red}{def update_edge_distance}(edge_distance: np.ndarray,
                         local_opt_routes: np.ndarray,
                         edge_n_used: np.ndarray,
                         demands: np.ndarray,
                         vehicle_capacity: int) -> np.ndarray:
    updated_edge_distance = np.copy(edge_distance)
    np.fill_diagonal(updated_edge_distance, 0.0)

    cap = max(vehicle_capacity, 1.0)
    eps = 1e-8

    route_load_factor = lambda demand: demand / (cap + eps)

    if local_opt_routes.size != 0:
        for route in local_opt_routes:
            customers = route[route >= 0]

            if customers.size > 0:
                route_demand = np.sum(demands[customers])
                route_ratio = route_load_factor(route_demand)

                route_edges = ( [(0, int(customers[0]))]
                    + [ (int(customers[i]), int(customers[i + 1]))
                        for i in range(len(customers) - 1)
                    ]
                    + [(int(customers[-1]), 0)]
                )

                unique_edges = set(route_edges)

                for u, v in unique_edges:
                    edge_usage = edge_n_used[u, v] + 1
                    combined_demand = ( demands[u] + demands[v] ) / (cap + eps)

                    penalty = ( 1.0
                        + edge_usage ** 1.5 * 0.12
                        + combined_demand * 0.2
                        + route_ratio * 0.4
                    )

                    adaptive_factor = 1.0
                    updated_penalty = penalty * adaptive_factor * 0.18
                    updated_edge_distance[u, v] += updated_penalty
                    updated_edge_distance[v, u] += updated_penalty

                    edge_n_used[u, v] += 1
                    edge_n_used[v, u] += 1

    updated_edge_distance = ( updated_edge_distance + updated_edge_distance.T ) / 2
    updated_edge_distance = np.clip(updated_edge_distance, 0, None)
    np.fill_diagonal(updated_edge_distance, 0.0)

    \textcolor{red}{return} updated_edge_distance
\end{Verbatim}

\end{tcolorbox}

\subsection{Structural Focus and Diversity Indicators}
\label{apx:structural-indicators}
\begingroup
\setlength{\abovedisplayskip}{16pt}
\setlength{\belowdisplayskip}{16pt}
\setlength{\abovedisplayshortskip}{10pt}
\setlength{\belowdisplayshortskip}{10pt}
\setlength{\parskip}{3pt}

Consider an evolved individual
\[
a_j=\bigl(Z_j,F_j^{s},F_j^{u}\bigr),
\]
where \(Z_j\) is the shared design blueprint,
\(F_j^{s}\) is the solution-initialization function, and
\(F_j^{u}\) is the penalty-construction function.

For a problem instance \(x\), let \(\mathcal{D}=(d_{ij})\) denote the original
distance matrix and let \(\theta_x\) collect the problem-specific
information. For TSP, \(\theta_x\) is omitted; for CVRP,
\(\theta_x\) contains the customer demands and vehicle capacity.
The initialization function constructs the initial solution
\begin{equation}
S_{0}(a_j;x)
=
F_j^{s}\!\left(\mathcal{D};\theta_x\right).
\label{eq:structural-initial-solution}
\end{equation}

Starting from \(S_{0}(a_j;x)\), local search and guided perturbation
produce a sequence of search states. Define
\(S_h(a_j;x)\) as the current search state at the \(h\)-th call to
\(F_j^{u}\), and define its route-edge collection as
\begin{equation}
E_h(a_j;x)
=
E\!\left(S_h(a_j;x)\right).
\label{eq:structural-edge-set}
\end{equation}

Let \(\mathcal{P}_h\) denote the edge-usage count matrix at the \(h\)-th update call. Applying \(F_j^{u}\) to the current search state defines the guided  distance matrix
\begin{equation}
\mathcal{D}_h^{\prime}
=
F_j^{u}\!\left(
\mathcal{D},
S_h(a_j;x),
\mathcal{P}_h;
\theta_x
\right),
\label{eq:structural-guided-matrix}
\end{equation}
where \(\mathcal{D}_h^{\prime}=(d_{ij,h}^{\prime})\).

Define the positive gap on edge \((i,j)\) at update call \(h\) as
\begin{equation}
\Delta_{ij,h}^{+}
=
\max\!\left\{
d_{ij,h}^{\prime}-d_{ij},
0
\right\}.
\label{eq:structural-positive-gap}
\end{equation}
Because of the symmetric routing problems considered in this study, define
the symmetric positive gap as
\begin{equation}
\overline{\Delta}_{ij,h}^{+}
=
\frac{1}{2}
\left(
\Delta_{ij,h}^{+}
+
\Delta_{ji,h}^{+}
\right).
\label{eq:structural-symmetric-gap}
\end{equation}

Define the average positive gap on the current search-state edges as
\begin{equation}
\mu_{\mathrm{path},h}(a_j;x)
=
\frac{1}{|E_h(a_j;x)|}
\sum_{(i,j)\in E_h(a_j;x)}
\overline{\Delta}_{ij,h}^{+}.
\label{eq:structural-path-mean}
\end{equation}

Let \(n_x\) denote the number of nodes in instance \(x\), and define the candidate undirected-edge set as
\[
\mathcal{E}_x
=
\left\{
(i,j):0\leq i<j<n_x
\right\}.
\]
The global average positive gap is then defined as
\begin{equation}
\mu_{\mathrm{all},h}(a_j;x)
=
\frac{1}{|\mathcal{E}_x|}
\sum_{(i,j)\in\mathcal{E}_x}
\overline{\Delta}_{ij,h}^{+}.
\label{eq:structural-global-mean}
\end{equation}

Define the Penalty Focus Ratio as
\begin{equation}
\mathrm{PFR}_{h}(a_j;x)
=
\frac{
\mu_{\mathrm{path},h}(a_j;x)
}{
\mu_{\mathrm{all},h}(a_j;x)+\varepsilon
},
\label{eq:structural-pfr}
\end{equation}
where \(\varepsilon>0\) is a small constant that prevents division
by zero. The normalized Focus indicator is defined as
\begin{equation}
\mathrm{Focus}_{h}(a_j;x)
=
\frac{
\max\!\left\{
0,\mathrm{PFR}_{h}(a_j;x)-1
\right\}
}{
1+
\max\!\left\{
0,\mathrm{PFR}_{h}(a_j;x)-1
\right\}
}.
\label{eq:structural-focus}
\end{equation}

To quantify edge-level differentiation within the current search
state, define the standard deviation of the positive gaps as
\begin{equation}
\sigma_{\mathrm{path},h}(a_j;x)
=
\sqrt{
\frac{1}{|E_h(a_j;x)|}
\sum_{(i,j)\in E_h(a_j;x)}
\left(
\overline{\Delta}_{ij,h}^{+}
-
\mu_{\mathrm{path},h}(a_j;x)
\right)^2
}.
\label{eq:structural-path-std}
\end{equation}

Define the Penalty Focus Diversity ratio and its normalized form as
\begin{equation}
\mathrm{PFD}_{h}(a_j;x)
=
\frac{
\sigma_{\mathrm{path},h}(a_j;x)
}{
\mu_{\mathrm{path},h}(a_j;x)+\varepsilon
},
\qquad
\mathrm{Diversity}_{h}(a_j;x)
=
\frac{
\mathrm{PFD}_{h}(a_j;x)
}{
1+\mathrm{PFD}_{h}(a_j;x)
}.
\label{eq:structural-diversity}
\end{equation}

\endgroup
\clearpage
\section{Performance Comparison on Public Benchmark Instances}
\label{app:public-benchmark-results}
\subsection{Results on TSPLIB Instances}
\label{apx:tsplib-results}

\begin{table}[H]
  \centering
  \caption{Performance comparison on TSPLIB instances}
  \label{tab:tsplib_gap}
 \footnotesize
  \begin{threeparttable}
    \setlength{\tabcolsep}{2.5pt}
    \resizebox{\linewidth}{!}{%
    \begin{tabular}{lcccccccccccccc}
      \toprule
      Instance
      & LS & GLS & FA-GLS & KGLS-r & KGLS-c & EBGLS
      & AM & POMO & LEHD & GNNGLS & NeuralGLS
      & EoH & ReEvo & LLM-HCJG \\
      \midrule
      eil51      & 1.214 & 1.643 & \textbf{0.000} & 0.700 & 0.700 & 0.700 & 1.628 & 0.829 & 1.640 & \textbf{0.000} & \textbf{0.000} & \textbf{0.000} & \textbf{0.000} & \textbf{0.000} \\
      berlin52   & 3.894 & \textbf{0.000} & \textbf{0.000} & 0.031 & 0.031 & 0.031 & 4.169 & 0.035 & 0.031 & 0.142 & \textbf{0.000} & 3.911 & \textbf{0.000} & \textbf{0.000} \\
      st70       & 1.945 & 0.741 & \textbf{0.000} & 0.313 & 0.313 & 0.325 & 1.737 & 0.313 & 0.325 & 0.764 & \textbf{0.000} & 1.185 & \textbf{0.000} & \textbf{0.000} \\
      eil76      & 4.146 & 0.558 & \textbf{0.000} & 1.184 & 1.184 & 1.184 & 1.992 & 1.184 & 2.544 & 0.163 & \textbf{0.000} & \textbf{0.000} & \textbf{0.000} & \textbf{0.000} \\
      pr76       & 1.506 & \textbf{0.000} & \textbf{0.000} & \textbf{0.000} & \textbf{0.000} & 0.097 & 0.816 & \textbf{0.000} & 0.219 & 0.039 & 0.823 & 1.505 & \textbf{0.000} & \textbf{0.000} \\
      rat99      & 6.459 & 4.211 & \textbf{0.000} & 0.681 & 0.826 & 6.459 & 2.645 & 2.392 & 1.099 & 0.550 & 0.718 & 7.597 & \textbf{0.000} & \textbf{0.000} \\
      kroA100    & 2.998 & 0.108 & \textbf{0.000} & 0.119 & 0.016 & 2.998 & 4.017 & 0.413 & 0.119 & 0.728 & 0.029 & 2.979 & \textbf{0.000} & \textbf{0.000} \\
      kroB100    & 0.577 & 0.262 & 0.262 & 0.254 & 0.254 & 0.577 & 5.142 & 0.323 & 0.256 & 0.147 & 0.882 & 0.578 & \textbf{0.000} & \textbf{0.000} \\
      kroC100    & 5.582 & 2.024 & \textbf{0.000} & 0.344 & 0.008 & 0.008 & 0.972 & 0.183 & 0.325 & 1.571 & 1.771 & 5.591 & \textbf{0.000} & \textbf{0.000} \\
      kroD100    & 6.641 & 0.545 & \textbf{0.000} & 0.001 & 0.001 & 0.001 & 2.717 & 0.842 & 0.383 & 0.572 & \textbf{0.000} & 6.631 & \textbf{0.000} & \textbf{0.000} \\
      kroE100    & 4.641 & 0.480 & 0.172 & 0.174 & 0.194 & 0.288 & 1.470 & 0.450 & 0.426 & 1.216 & 1.053 & 4.649 & \textbf{0.000} & \textbf{0.000} \\
      rd100      & 9.314 & 0.013 & \textbf{0.000} & 0.005 & 0.005 & 0.202 & 3.407 & 0.005 & 0.005 & 0.459 & \textbf{0.000} & 9.317 & \textbf{0.000} & \textbf{0.000} \\
      eil101     & 5.180 & 0.477 & \textbf{0.000} & 2.113 & 2.514 & 2.117 & 2.994 & 1.844 & 2.309 & 0.201 & 0.362 & \textbf{0.000} & \textbf{0.000} & \textbf{0.000} \\
      lin105     & 2.606 & 0.730 & \textbf{0.000} & 0.028 & 0.028 & 0.313 & 1.739 & 0.521 & 0.343 & 0.606 & 0.647 & 2.573 & \textbf{0.000} & \textbf{0.000} \\
      pr107      & 0.614 & 0.305 & \textbf{0.000} & \textbf{0.000} & \textbf{0.000} & 0.414 & 3.933 & 0.524 & 11.235 & 0.439 & 0.808 & 0.618 & \textbf{0.000} & \textbf{0.000} \\
      pr124      & 2.441 & 0.878 & \textbf{0.000} & 2.441 & 2.441 & 0.650 & 3.677 & 0.603 & 1.112 & 0.755 & 0.075 & 2.439 & \textbf{0.000} & \textbf{0.000} \\
      bier127    & 1.789 & 0.579 & 0.119 & 0.190 & 0.185 & 0.247 & 5.908 & 13.721 & 4.761 & 1.948 & 2.272 & 1.777 & 0.037 & \textbf{0.026} \\
      ch130      & 7.613 & 0.295 & 0.393 & 0.012 & 0.012 & 0.068 & 3.182 & 0.157 & 0.554 & 3.519 & 1.188 & 6.579 & \textbf{0.000} & \textbf{0.000} \\
      pr136      & 6.304 & 1.898 & 0.234 & 0.155 & 0.092 & 0.152 & 5.064 & 0.927 & 0.449 & 3.387 & 2.319 & 6.305 & \textbf{0.000} & \textbf{0.000} \\
      pr144      & 4.194 & 0.386 & \textbf{0.000} & \textbf{0.000} & \textbf{0.000} & \textbf{0.000} & 7.641 & 0.531 & 0.194 & 3.581 & 0.743 & 4.196 & \textbf{0.000} & \textbf{0.000} \\
      ch150      & 1.282 & 0.934 & \textbf{0.000} & 0.437 & 0.429 & 0.429 & 4.584 & 0.528 & 0.519 & 2.113 & 2.488 & 2.099 & \textbf{0.000} & \textbf{0.000} \\
      kroA150    & 7.101 & 3.273 & 0.004 & 0.099 & 0.099 & 1.458 & 3.784 & 0.696 & 1.402 & 2.984 & 0.773 & 7.099 & \textbf{0.000} & \textbf{0.000} \\
      kroB150    & 5.546 & 2.748 & 0.758 & 0.790 & 1.234 & 0.091 & 2.437 & 1.167 & 0.756 & 3.258 & 3.114 & 4.298 & 0.008 & \textbf{0.000} \\
      pr152      & 1.895 & 0.185 & 0.185 & 1.895 & 1.895 & 0.189 & 7.494 & 1.054 & 12.136 & 3.119 & \textbf{0.000} & 1.889 & \textbf{0.000} & \textbf{0.000} \\
      u159       & 5.952 & 1.457 & 0.580 & \textbf{0.000} & \textbf{0.000} & 0.150 & 7.551 & 0.951 & 1.132 & 1.020 & 0.904 & 5.784 & \textbf{0.000} & \textbf{0.000} \\
      rat195     & 1.562 & 1.722 & 0.387 & 1.562 & 1.106 & 0.640 & 6.893 & 8.150 & 1.418 & 1.666 & 0.478 & 1.894 & 0.430 & \textbf{0.215} \\
      d198       & 1.775 & 0.817 & 0.114 & 0.393 & 0.425 & 0.824 & 373.020 & 17.290 & 9.235 & 4.772 & 1.280 & 1.027 & 0.114 & \textbf{0.101} \\
      kroA200    & 0.907 & 0.872 & \textbf{0.000} & 0.548 & 0.459 & 0.567 & 7.106 & 1.577 & 0.644 & 2.029 & 0.861 & 0.916 & 0.177 & 0.089 \\
      kroB200    & 5.357 & 4.308 & 0.856 & 0.722 & 0.984 & 1.524 & 8.541 & 1.440 & 0.156 & 2.589 & 3.742 & 5.364 & 0.099 & \textbf{0.051} \\
      \midrule
      Average    & 3.829 & 1.119 & 0.140 & 0.524 & 0.532 & 0.783 & 16.768 & 2.022 & 1.922 & 1.529 & 0.942 & 3.407 & 0.030 & \textbf{0.017} \\
      Number     & 0 & 2 & 17 & 4 & 4 & 1 & 0 & 1 & 0 & 1 & 7 & 3 & 23 & \textbf{28} \\
      \bottomrule
    \end{tabular}
    }

    \begin{tablenotes}[flushleft]
      \footnotesize
      \item[] \textbf{Notes:}
      All entries are optimality gaps (\%). Boldface marks the
      smallest or tied-smallest gap on each instance.
    \end{tablenotes}
  \end{threeparttable}
\end{table}
\vspace{-0.15em}
\subsection{Results on CVRPLIB Instances}
\label{apx:cvrplib-results}

\begin{table}[H]
  \centering
  \caption{Performance comparison on CVRPLIB instances}
  \label{tab:cvrplib_gap}
  \footnotesize
  \renewcommand{\arraystretch}{0.82}
  \begin{threeparttable}
    \setlength{\tabcolsep}{5.5pt}
    \begin{tabular}{lccccc}
      \toprule
      Instance & GLS & FA-GLS & EoH & ReEvo & LLM-HCJG \\
      \midrule
      A-n32-k5   & \textbf{0.000} & \textbf{0.000} & \textbf{0.000} & \textbf{0.000} & \textbf{0.000} \\
      A-n60-k9   & 0.369 & 0.812 & 0.369 & 7.164 & \textbf{0.295} \\
      B-n50-k7   & 8.637 & 0.405 & \textbf{0.000} & 0.270 & \textbf{0.000} \\
      B-n78-k10  & 11.138 & 3.440 & 1.229 & 6.798 & \textbf{0.082} \\
      E-n76-k10  & 3.614 & 7.349 & 4.337 & 14.819 & \textbf{1.687} \\
      E-n101-k8  & 6.258 & 3.436 & 2.945 & 6.012 & \textbf{0.491} \\
      P-n55-k15  & 15.875 & 19.818 & 10.313 & 35.490 & \textbf{4.449} \\
      P-n101-k4  & 1.028 & 8.370 & 2.496 & 6.314 & \textbf{0.294} \\
      M-n151-k12 & 11.133 & 7.783 & 4.532 & 12.709 & \textbf{3.054} \\
      X-n162-k11 & 5.029 & 6.613 & 4.994 & 6.882 & \textbf{1.457} \\
      X-n167-k10 & 7.219 & 8.902 & 6.990 & 10.702 & \textbf{4.461} \\
      X-n190-k8  & 4.576 & 6.896 & 5.012 & 6.254 & \textbf{3.963} \\
       \midrule
      Average & 6.240 & 6.152 & 3.601 & 9.451 & \textbf{1.686} \\
      Number  & 1 & 1 & 2 & 1 & \textbf{12} \\
      \bottomrule
    \end{tabular}
  \end{threeparttable}
\end{table}
\vspace{-0.6em}
\clearpage
\section{Joint Component Co-Evolution Ablation}
\vspace{-0.4em}
\label{app:joint-coevolution-ablation}

\subsection{Cross-Component Combination Results on Synthetic TSP Instances}
\label{apx:ablation-synthetic-tsp}
\begin{table}[H]
  \centering
  \caption{Cross-component combination ablation results on synthetic TSP instances.}
  \label{tab:ablation_synthetic_tsp}
  \fontsize{10.8}{11.5}\selectfont
  \renewcommand{\arraystretch}{1.10}
  \begin{threeparttable}
    \setlength{\tabcolsep}{6pt}
    \begin{tabular}{lcccccc}
      \toprule
      \multirow{2}{*}{Combination} &
      \multicolumn{2}{c}{\textbf{TSP20}} &
      \multicolumn{2}{c}{\textbf{TSP50}} &
      \multicolumn{2}{c}{\textbf{TSP100}} \\
      \cmidrule(lr){2-3}
      \cmidrule(lr){4-5}
      \cmidrule(lr){6-7}
      & Gap (\%) & Time (s)
      & Gap (\%) & Time (s)
      & Gap (\%) & Time (s) \\
      \midrule
      \texttt{EoH-EoH}
      & \textbf{0.000} & 3.044
      & 0.001 & 7.449
      & 0.159 & 22.505 \\

      \texttt{HCJG-HCJG}
      & \textbf{0.000} & 1.757
      & \textbf{0.000} & 2.960
      & \textbf{0.070} & 5.527 \\

      \texttt{EoH-HCJG}
      & \textbf{0.000} & 1.934
      & \textbf{0.000} & 2.830
      & 0.082 & 5.059 \\

      \texttt{HCJG-EoH}
      & \textbf{0.000} & 5.209
      & \textbf{0.000} & 8.750
      & 0.152 & 26.142 \\
      \bottomrule
    \end{tabular}
  \end{threeparttable}
\end{table}
\vspace{-1.0em}
\subsection{Cross-Component Combination Results on TSPLIB Instances}
\label{apx:ablation-tsplib}

\begin{table}[H]
  \centering
  \caption{Cross-component combination ablation results on TSPLIB instances.}
  \label{tab:ablation_tsplib}
  \fontsize{10.2}{10.8}\selectfont
  \renewcommand{\arraystretch}{0.99}
  \begin{threeparttable}
    \setlength{\tabcolsep}{6pt}
    \begin{tabular}{lcccc}
      \toprule
      Instance &
      \texttt{HCJG-HCJG} &
      \texttt{EoH-EoH} &
      \texttt{EoH-HCJG} &
      \texttt{HCJG-EoH} \\
      \midrule
      eil51    & \textbf{0.000} & \textbf{0.000} & \textbf{0.000} & \textbf{0.000} \\
      berlin52 & \textbf{0.000} & 2.294 & \textbf{0.000} & \textbf{0.000} \\
      st70     & \textbf{0.000} & \textbf{0.000} & \textbf{0.000} & \textbf{0.000} \\
      eil76    & \textbf{0.000} & \textbf{0.000} & \textbf{0.000} & \textbf{0.000} \\
      pr76     & \textbf{0.000} & 5.504 & \textbf{0.000} & \textbf{0.000} \\
      rat99    & \textbf{0.000} & 3.468 & \textbf{0.000} & 1.156 \\
      kroA100  & \textbf{0.000} & 1.198 & \textbf{0.000} & \textbf{0.000} \\
      kroB100  & \textbf{0.000} & 5.090 & \textbf{0.000} & 0.357 \\
      kroC100  & \textbf{0.000} & 3.634 & \textbf{0.000} & 0.713 \\
      kroD100  & \textbf{0.000} & 6.490 & \textbf{0.000} & 1.371 \\
      kroE100  & \textbf{0.000} & 2.411 & \textbf{0.000} & 0.195 \\
      rd100    & \textbf{0.000} & 1.820 & \textbf{0.000} & 0.367 \\
      eil101   & \textbf{0.000} & \textbf{0.000} & \textbf{0.000} & \textbf{0.000} \\
      lin105   & \textbf{0.000} & 2.170 & \textbf{0.000} & \textbf{0.000} \\
      pr107    & \textbf{0.000} & 0.618 & \textbf{0.000} & \textbf{0.000} \\
      pr124    & \textbf{0.000} & 1.896 & \textbf{0.000} & 0.601 \\
      bier127  & \textbf{0.026} & 3.817 & 0.063 & 0.965 \\
      ch130    & \textbf{0.000} & 5.385 & \textbf{0.000} & 1.178 \\
      pr136    & \textbf{0.000} & 0.339 & \textbf{0.000} & 1.406 \\
      pr144    & \textbf{0.000} & 1.280 & \textbf{0.000} & \textbf{0.000} \\
      ch150    & \textbf{0.000} & 3.676 & \textbf{0.000} & 0.077 \\
      kroA150  & \textbf{0.000} & 6.466 & \textbf{0.000} & 1.218 \\
      kroB150  & \textbf{0.000} & 4.057 & \textbf{0.000} & 1.382 \\
      pr152    & \textbf{0.000} & 3.644 & \textbf{0.000} & 0.471 \\
      u159     & \textbf{0.000} & 0.530 & \textbf{0.000} & 0.960 \\
      rat195   & \textbf{0.215} & 3.874 & 0.517 & 1.378 \\
      d198     & \textbf{0.101} & 0.627 & 0.133 & 0.754 \\
      kroA200  & \textbf{0.089} & 2.874 & 0.099 & 1.199 \\
      kroB200  & \textbf{0.051} & 7.698 & 0.065 & 1.981 \\
      \midrule
      Mean gap & \textbf{0.017} & 2.788 & 0.030 & 0.611 \\
      \bottomrule
    \end{tabular}
  \end{threeparttable}
\end{table}
\begin{figure}[H]
  \centering
  \includegraphics[width=0.98\textwidth]
  {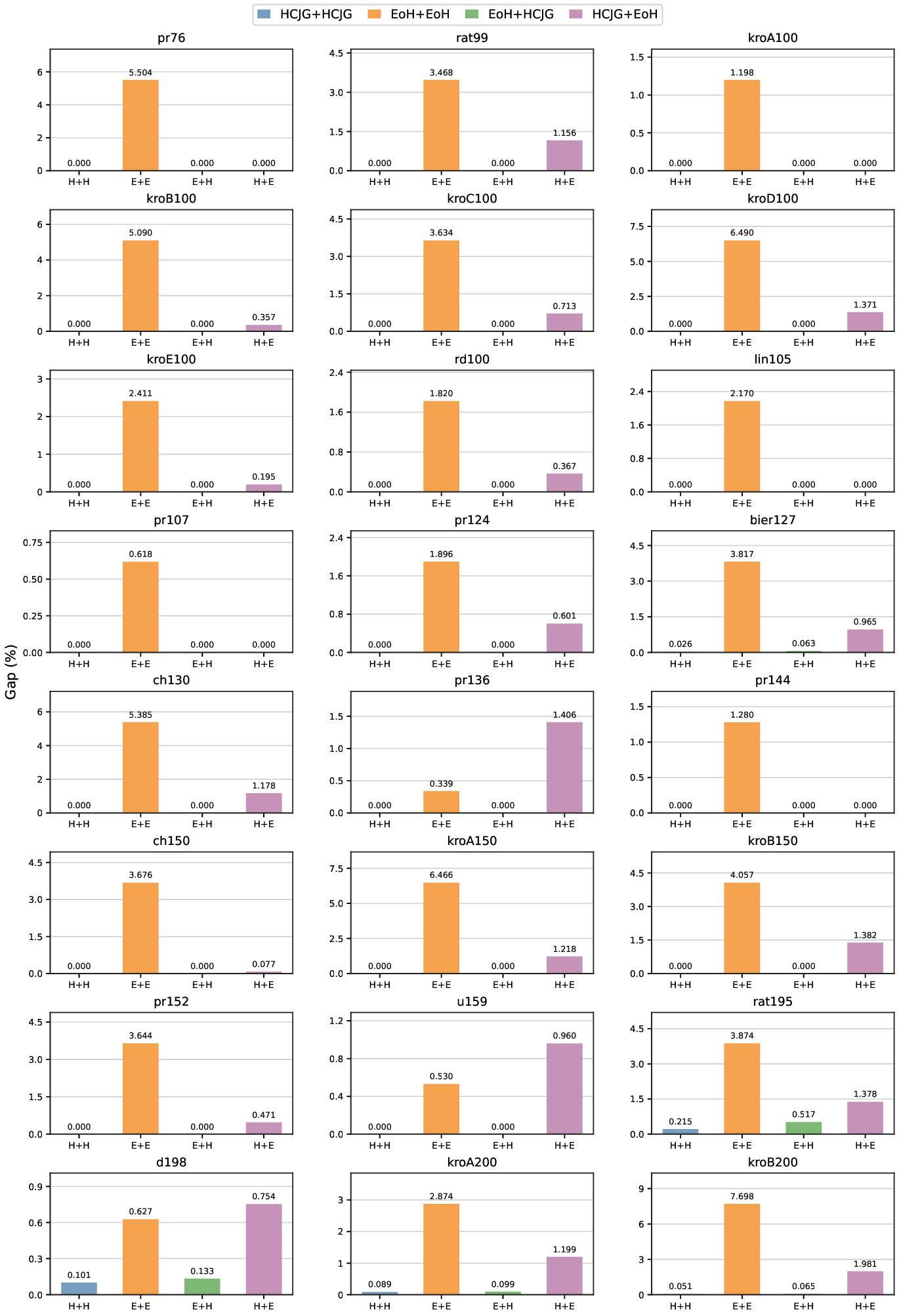}
  \caption{Visual comparison of optimality gaps on TSPLIB with non-identical results.}
  \label{fig:ablation-tsplib-nonidentical-gaps}
\end{figure}

\section{Statistical Significance and Uncertainty Analysis}
\label{app:statistical-significance}

Paired analyses are conducted using the per-instance optimality gaps reported in Appendices~\ref{apx:tsplib-results}, \ref{apx:cvrplib-results}, and~\ref{apx:ablation-tsplib}, with paired differences defined as the compared-method gap minus the HCJG gap, in percentage points. Uncertainty is quantified by percentile-bootstrap 95\% confidence intervals for the mean paired difference using 20,000 instance-level resamples with a fixed random seed, while statistical significance is assessed using two-sided Wilcoxon signed-rank tests at $\alpha=0.05$.

\subsection{Public TSPLIB and CVRPLIB Benchmarks}
\label{apx:stat-public}

\begin{table}[H]
\centering
\caption{Paired uncertainty and significance analysis on public benchmarks.}
\label{tab:stat-public}
\begin{threeparttable}
\begin{tabular}{llccc}
\toprule
Benchmark & Comparator & $n$ & Bootstrap 95\% CI (pp) & Wilcoxon $p$-value \\
\midrule
TSPLIB  & GLS     & 29 & [0.703, 1.547] & $<0.0001$ \\
TSPLIB  & FA-GLS  & 29 & [0.048, 0.213] & 0.0030 \\
TSPLIB  & EoH     & 29 & [2.476, 4.343] & $<0.0001$ \\
TSPLIB  & ReEvo   & 29 & [0.001, 0.031] & 0.0277 \\
\midrule
CVRPLIB & GLS     & 12 & [2.306, 6.939] & 0.0010 \\
CVRPLIB & FA-GLS  & 12 & [2.467, 6.936] & 0.0010 \\
CVRPLIB & EoH     & 12 & [1.052, 2.879] & 0.0020 \\
CVRPLIB & ReEvo   & 12 & [4.146, 12.759] & 0.0010 \\
\bottomrule
\end{tabular}
\begin{tablenotes}[flushleft]
\footnotesize
\item[] \textbf{Notes:} The interval entirely above zero favors  LLM-HCJG.
\end{tablenotes}
\end{threeparttable}
\end{table}

\subsection{Cross-Component Co-Evolution Ablation}
\label{apx:stat-ablation}

\begin{table}[H]
\centering
\caption{Paired uncertainty and significance analysis for TSPLIB ablation.}
\label{tab:stat-ablation}
\begin{threeparttable}
\begin{tabular}{lccc}
\toprule
Comparator & $n$ & Bootstrap 95\% CI (pp) & Wilcoxon $p$-value \\
\midrule
\texttt{EoH-EoH}   & 29 & [2.010, 3.575] & $<0.0001$ \\
\texttt{EoH-HCJG}  & 29 & [0.001, 0.036] & 0.0431 \\
\texttt{HCJG-EoH}  & 29 & [0.388, 0.801] & 0.0001 \\
\bottomrule
\end{tabular}
\begin{tablenotes}[flushleft]
\footnotesize
\item[] \textbf{Note:} The interval entirely above zero favors \texttt{HCJG+HCJG}.
\end{tablenotes}
\end{threeparttable}
\end{table}

\section{Proofs of Theoretical Results}
\label{app:theoretical-proofs}

\subsection{Proof of Theorem~\ref{thm:nonseparability}}
\label{app:proof-nonseparability}
\begin{proof}
Let
\begin{equation}
A^{s}=A\big(S(f^{s})\big),
\qquad
\widetilde{A}^{s}=A\big(S(\tilde f^{s})\big),
\end{equation}
and
\begin{equation}
B^{u}=B_{f^{u}}(D,P),
\qquad
\widetilde{B}^{u}=B_{\tilde f^{u}}(D,P).
\end{equation}
By \eqref{eq:route-supported-penalty},
\begin{align}
\mathcal{U}_{f^{u}}\big(D,S(f^{s}),P\big)&=D+A^{s}\odot B^{u},\\
\mathcal{U}_{f^{u}}\big(D,S(\tilde f^{s}),P\big)&=D+\widetilde{A}^{s}\odot B^{u},\\
\mathcal{U}_{\tilde f^{u}}\big(D,S(f^{s}),P\big)&=D+A^{s}\odot \widetilde{B}^{u},\\
\mathcal{U}_{\tilde f^{u}}\big(D,S(\tilde f^{s}),P\big)&=D+\widetilde{A}^{s}\odot \widetilde{B}^{u}.
\end{align}
Substituting these four identities into \eqref{eq:component-interaction} gives
\begin{align}
\mathfrak{I}
={}&\left(D+A^{s}\odot B^{u}\right)-\left(D+\widetilde{A}^{s}\odot B^{u}\right)\nonumber\\
&-\left(D+A^{s}\odot \widetilde{B}^{u}\right)+\left(D+\widetilde{A}^{s}\odot \widetilde{B}^{u}\right)\nonumber\\
={}&\left(A^{s}-\widetilde{A}^{s}\right)\odot\left(B^{u}-\widetilde{B}^{u}\right).
\label{eq:interaction-factorization}
\end{align}
By \eqref{eq:path-difference-condition} and \eqref{eq:penalty-difference-condition}, there exists an edge $(i,j)$ for which
\begin{equation}
\left[A^{s}-\widetilde{A}^{s}\right]_{ij}\neq 0
\qquad\text{and}\qquad
\left[B^{u}-\widetilde{B}^{u}\right]_{ij}\neq 0.
\end{equation}
It follows from \eqref{eq:interaction-factorization} that
\begin{equation}
\mathfrak{I}_{ij}=
\left[A^{s}-\widetilde{A}^{s}\right]_{ij}
\left[B^{u}-\widetilde{B}^{u}\right]_{ij}\neq 0,
\end{equation}
and hence $\mathfrak{I}\neq 0$.

It remains to prove that \eqref{eq:additive-decomposition} cannot hold. Suppose, to the contrary, that there exist $\mathcal{G}^{s}$ and $\mathcal{G}^{u}$ satisfying \eqref{eq:additive-decomposition} for every component combination. By Definition~\ref{def:component-interaction},
\begin{align}
\mathfrak{I}
={}&\left[\mathcal{G}^{s}(f^{s};D,P)+\mathcal{G}^{u}(f^{u};D,P)\right]
-\left[\mathcal{G}^{s}(\tilde f^{s};D,P)+\mathcal{G}^{u}(f^{u};D,P)\right]\nonumber\\
&-\left[\mathcal{G}^{s}(f^{s};D,P)+\mathcal{G}^{u}(\tilde f^{u};D,P)\right]
+\left[\mathcal{G}^{s}(\tilde f^{s};D,P)+\mathcal{G}^{u}(\tilde f^{u};D,P)\right]\\
={}&0,
\end{align}
which contradicts $\mathfrak{I}\neq 0$. Therefore, \eqref{eq:additive-decomposition} is impossible, and the initialization and penalty-construction components are non-separable at the search-state-transition level.
\end{proof}

\subsection{Proof of Theorem~\ref{thm:blueprint}}
\label{app:proof-blueprint}
\begin{proof}
For joint generation, \eqref{eq:blueprint-faithfulness} implies that, conditional on any blueprint $Z=z$,
\begin{equation}
\zeta^{s}(F^{s,J})=z,
\qquad
\zeta^{u}(F^{u,J})=z
\end{equation}
hold simultaneously with probability one. Hence,
\begin{equation}
\Pr\!\left((F^{s,J},F^{u,J})\in\mathcal{C}\mid Z=z\right)=1,
\qquad \forall z\in\mathcal{Z}.
\end{equation}
Applying the law of total probability yields
\begin{align}
\Pr\!\left((F^{s,J},F^{u,J})\in\mathcal{C}\right)
&=\sum_{z\in\mathcal{Z}}\Pr(Z=z)\Pr\!\left((F^{s,J},F^{u,J})\in\mathcal{C}\mid Z=z\right)\nonumber\\
&=\sum_{z\in\mathcal{Z}}\pi_{\theta}(z\mid x)=1,
\end{align}
which proves \eqref{eq:joint-consistency-probability}.

For independent generation, \eqref{eq:independent-blueprints} and \eqref{eq:blueprint-faithfulness} imply that
\begin{equation}
(F^{s,I},F^{u,I})\in\mathcal{C}
\Longleftrightarrow Z_s=Z_u.
\end{equation}
Therefore,
\begin{align}
\Pr\!\left((F^{s,I},F^{u,I})\in\mathcal{C}\right)
&=\Pr(Z_s=Z_u)\nonumber\\
&=\sum_{z\in\mathcal{Z}}\Pr(Z_s=z,Z_u=z)\nonumber\\
&=\sum_{z\in\mathcal{Z}}\pi_{\theta}(z\mid x)^2,
\end{align}
where the last equality follows from the independence of $Z_s$ and $Z_u$. Taking the complement proves \eqref{eq:mismatch-probability}.

Finally, if there exist distinct $z_1,z_2\in\mathcal{Z}$ such that
\begin{equation}
\pi_{\theta}(z_1\mid x)>0,
\qquad
\pi_{\theta}(z_2\mid x)>0,
\end{equation}
then $\sum_{z\in\mathcal{Z}}\pi_{\theta}(z\mid x)^2<1$. Thus, the mismatch probability in \eqref{eq:mismatch-probability} is strictly positive.
\end{proof}

\end{document}